\documentclass{article}

\usepackage[final,main]{neurips_style/neurips_2026}

\usepackage[utf8]{inputenc}
\usepackage[T1]{fontenc}
\usepackage[hypertexnames=false]{hyperref}
\usepackage{url}
\usepackage{booktabs}
\usepackage{amsfonts}
\usepackage{amsmath}
\usepackage{amssymb}
\usepackage{nicefrac}
\usepackage{microtype}
\usepackage{xcolor}
\usepackage{graphicx}
\usepackage{subcaption}
\IfFileExists{svg.sty}{\usepackage{svg}}{}
\graphicspath{{figures/}}
\usepackage{float}
\usepackage{algorithm}

\usepackage{algpseudocode}
\usepackage{tikz}
\usetikzlibrary{positioning,calc,decorations.pathreplacing}
\usepackage{amsthm}

\newtheorem{proposition}{Proposition}

\newtheorem{remark}[proposition]{Remark}

\usetikzlibrary{positioning, arrows.meta, calc, fit, shapes.geometric}

\title{Exact Unlearning via Quantized Sufficient Statistics}

\author{
  Ami Tavory \\
  Meta Platforms, Inc.
  \And
  Shripad Gade \\
  Meta Platforms, Inc.
  \And
  Tal Sarig \\
  Meta Platforms, Inc.
  \AND
  Noam Touitou \\
  Meta Platforms, Inc.
  \And
  Ido Guy \\
  Meta Platforms, Inc.
}

\begin{document}

\maketitle

\begin{abstract}
Exact unlearning requires a deployed predictor to match one rebuilt without the
information named by a deletion request. Existing general-purpose exact methods
localize retraining through disjoint shards, but every request still invalidates
a model, and smaller shards reduce the data available to each constituent
predictor. We introduce \emph{Quantized Sufficient Statistics} (QSS), which
separates a small frozen \emph{schema} from mutable, sum-decomposable
\emph{content}. The schema learns global structure; the content stores local
prediction corrections as additive statistics indexed by quantized regions.
Deleting content is therefore exact subtraction rather than optimization. We
distinguish two guarantees: QSS-L exactly removes a label while retaining the
unlabelled input, whereas QSS-E exactly removes both input and label by learning
the schema without deletable examples. A deletion takes the arithmetic fast path
with probability $1-\rho$ and triggers a full rebuild with probability $\rho$;
all reported expected latencies include both events. Across 15 vision, text, and
tabular datasets at $\rho=0.5\%$, QSS-L is within 2 percentage points of SISA on
11 tasks and provides 4--483$\times$ lower expected deletion latency on the
low-class-count tasks where a compact schema is effective. QSS-E quantifies the
additional accuracy cost of removing every trace of an input.
\end{abstract}

\section{Introduction}
\label{sec:intro}
\footnotetext{Code: \url{https://github.com/atavory/rqu}}

Machine unlearning has become a legal and operational requirement for deployed
models: privacy regulations such as the EU's
GDPR~\cite{mantelero2013eu} and California's
CCPA~\cite{ccpa_unlearning} mandate the ability to delete the
effect of specific training instances, while security and safety
concerns (removing poisoned~\cite{schoepf2024potion}, biased or
outdated data points~\cite{liu2025mitigating}) provide independent
motivation. Exactness is defined relative to the information named by the
request. We study two neighboring-dataset relations. \emph{Label deletion}
retains an input~$x$ but removes its label~$y$ and all supervised influence;
this models retracted annotations, outcomes, or responses when retaining the
unlabelled input is permitted. \emph{Instance deletion} removes both~$x$ and~$y$,
so the result must match training as if the example had never been present.
We write $\mathcal{D}^{-y_j}$ for the former neighboring dataset and
$\mathcal{D}^{-j}$ for the latter.
The fundamental obstacle in either case is \emph{entanglement}: standard
training weaves examples into shared parameters, leaving retraining as the
generic way to produce the appropriate neighboring-dataset model.

\begin{figure}[t!]
\centering
\includegraphics[width=\linewidth]{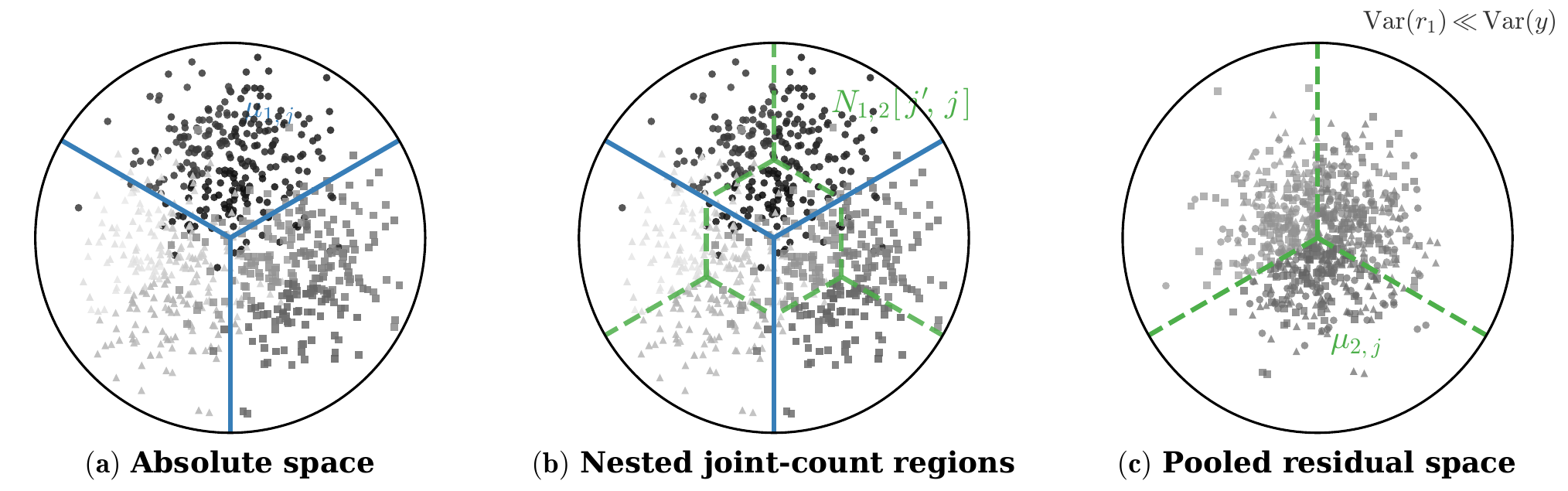}
\caption{\textbf{QSS translates Residual Quantization on embeddings to
residual boosting on labels.}
Marker shape ($\circ$, $\triangle$, $\square$) encodes cell assignment;
intensity encodes the target~$y_i$ (dark${=}$high, light${=}$low).
\textbf{(a)}~Level~1 Voronoi cells ($b{=}3$) partition the embedding
space; each cell has a distinct intensity level with visible internal spread.
\textbf{(b)}~A nested Level~2 codebook subdivides each cell; the
intersection regions define the co-occurrence counts $N_{1,2}[j',j]$
stored as sum-decomposable statistics, enabling exact unlearning via
simple count decrements.
\textbf{(c)}~After subtracting Level~1 cell means, all points collapse
into the pooled residual space. The wide intensity spread collapses to a narrow, uniform band,
visualizing that label residuals carry far less
variance than the original labels, since each RQ level corrects the previous,
analogous to gradient boosting~\cite{friedman2001greedy} with
quantization centroids as weak learners.}
\label{fig:architecture}
\label{fig:geometry}
\end{figure}

The standard exact baseline, SISA~\cite{bourtoule2021machine},
attacks this problem by \emph{disjoint partitioning}: it partitions the
dataset (size $n$) into $S$~shards, and trains a sub-model on each, simultaneously bounding the
retraining cost to~$O(n/S)$, but also ensuring that every deletion \emph{deterministically} invalidates a single shard
and forces an eventual retrain.  Under a batch of $B$~deletions,
the expected number of affected shards is
$S\bigl(1{-}(1{-}1/S)^B\bigr)$, which approaches~$S$ once
$B \gtrsim S\log S$.  SISA cannot lower its invalidation rate
without proportionally reducing the data available to each
shard; redundancy multiplies compute overhead without improving
availability.

We propose \textbf{Quantized Sufficient Statistics (QSS)}, which
breaks this coupling by changing the model representation. The intuition is a
wall of labelled bins. A frozen quantizer assigns similar inputs to bins, and
each bin stores only sums and counts of prediction errors. Prediction consults
the relevant bins; deleting an example subtracts its contribution. Recomputing
the same bin totals from the remaining information gives exactly the same
state, without rerunning an optimizer or scanning the dataset.

QSS realizes this idea by separating \emph{schema} from \emph{content}.
A small random subset $\mathcal{D}^{\circ}$ (size $n^{\circ}$, $\rho = n^{\circ}/n$,
typically ${\leq}\,1\%$) trains a frozen base model~$g$
and a Residual
Quantization (RQ)~\cite{chen2010approximate} codebook, the
\emph{schema} (see Table~\ref{tab:notation} for symbols).  The remaining data
$\mathcal{D}^{\bullet} = \mathcal{D} \setminus \mathcal{D}^{\circ}$
is pure \emph{content}: stored as sum-decomposable accumulators
indexed by RQ codes over frozen embeddings
(Figure~\ref{fig:architecture}).
Because the data structure \emph{is} the model, ordinary unlearning is
subtraction: deleting content is a constant-time arithmetic decrement,
independent of~$n$. QSS-L trains the label-blind codebook on all inputs and
supports exact label deletion. QSS-E trains the codebook only on the retained
schema and supports exact instance deletion for content examples. In either
mode, deleting a retained schema example invokes a full rebuild. Thus QSS has a
Las Vegas-style runtime guarantee: the result is always exact for the stated
neighbor relation, while the randomized schema determines whether the request
takes the arithmetic or rebuild path.
SISA cannot exploit partition redundancy, as overlapping shards
would simply require multiple partition rebuilds per deletion.
The probabilistic schema--content separation does not suffer the
same way.
We therefore maintain $p$~independent instances, each with its own random
$\mathcal{D}^{\circ}_i$, base model ($g_i$ trained on $\mathcal{D}^{\circ}_i$), and codebook.  A deletion
triggers the~$O(1)$ fast path for all active instances with
probability~$(1{-}\rho)^p$; a synchronous rebuild must occur only when
\emph{all}~$p$ instances are invalidated, an event with probability
$\rho^p$.  Every instance aggregates all
of~$\mathcal{D}$ in its accumulators, so while a pair rebuilds,
the remaining pairs continue serving at full accuracy.  In SISA,
taking a shard offline for retraining removes $1/S$ of the data
from the ensemble; in QSS, every pair independently sees all data,
so no accuracy is lost during rebuilds.

\smallskip
\noindent\textbf{Contributions.}
\begin{enumerate}\setlength\itemsep{2pt}
\item \textbf{Schema--content separation for exact unlearning}
  (\S\ref{sec:data_structure}):
  The schema--content separation proposed in this work reduces ordinary
  label or instance deletion to constant-time arithmetic updates, independent
  of~$n$ and identical to rebuilding for the corresponding neighbor relation.
  Co-occurrence counts avoid exponential fragmentation across
  quantization levels (Equation~\ref{eq:materialize}), leading
  to efficient polynomial-time re-materialization.

\item \textbf{Redundant schemas for amortized $O(1)$ deletion}
  (\S\ref{sec:amortized}):
  $p$~independent copies reduce the synchronous-rebuild probability
  to~$\rho^p$, with zero accuracy degradation.
  Under batch deletions of size~$B$, QSS forces a rebuild with
  probability $(1{-}(1{-}\rho)^B)^p$, negligible for
  typical~$\rho$ and~$p$.
\end{enumerate}

\smallskip
\noindent\textbf{Related work.}
\emph{Exact retraining methods.}
SISA~\cite{bourtoule2021machine} partitions data into shards and
retrains the affected shard per deletion at $O(n/S)$ cost.
Exact $k$-means deletion~\cite{ginart2019making} and exact random
forest unlearning via sufficient statistics~\cite{brophy2021machine}
achieve $n$-independent deletion for specific model families, but
rely on fixed, single-level partitions.
\cite{muresanu2025fast} freeze a pretrained LLM and isolate
mutable state in $k$-means clusters for $O(1)$ exact unlearning,
but require a full LLM forward pass per prediction; their clusters
suffer the same codebook instability we formalize in
Appendix~\ref{app:fragility}.
QSS extends sufficient-statistics unlearning to a learned,
multi-level quantized partition: the co-occurrence correction
in \textsc{RQA.Materialize} (Equation~\ref{eq:materialize}) is
what preserves decomposability across residual levels; without
it, level interactions would require joint $b^\ell$-cell
bookkeeping.

\smallskip
\noindent\emph{Analytic sufficient-statistic updates.}
Linear and ridge heads admit exact deletion by maintaining $X^\top X$ and
$X^\top Y$ and applying a rank-one downdate. ACU~\cite{tang2025acu} and the
centralized form of Quan et al.~\cite{quan2026exact} develop this analytic
approach for frozen representations. These methods are natural complements to
QSS: analytic updates are preferable when a modest-dimensional linear head is
accurate, whereas QSS stores nonlinear local residual structure in quantized
cells. The schema mechanism and the choice of analytic versus quantized residual
statistics are composable; Appendix~\ref{app:analytic} gives matched comparisons
and their different state/update costs.

\smallskip
\noindent\emph{Approximate methods.}
Certified removal~\cite{guo2020certified}, influence
functions~\cite{koh2017understanding}, Fisher
scrubbing~\cite{golatkar2020eternal}, and gradient
ascent~\cite{chen2025provable} provide varying degrees of
approximation; see surveys~\cite{nguyen2022survey,cheng2024mubench}.

\smallskip
\noindent\emph{Quantization and unlearning.}
QUAIL~\cite{mishra2026quail}, LUNAR~\cite{lunar2025}, and
OEU~\cite{nguyen2026oeu} treat quantization as a complication for
unlearning, treating quantization artifacts as problems to be mitigated. QSS inverts this
relationship: quantization \emph{is} the architecture, and the
quantized representation is what makes exact deletion tractable.

\input{fig_pipeline}
\section{Method}
\label{sec:method}
\label{sec:data_structure}  

\begin{table}[t]
\centering
\caption{Notation.  Superscript ${\circ}$ marks the retained subset;
  ${\bullet}$ marks the deletable portion.}
\label{tab:notation}
\smallskip
\begin{small}
\begin{tabular}{cl}
\toprule
Symbol & Meaning \\
\midrule
$X$, $Y$; $x_i$, $y_i$ & Inputs, labels; single instance, single label \\
$n$, $d$, $k$ & Dataset size, embedding dimensionality, label dimensionality \\
$f$ & Frozen upstream encoder (CLIP, SBERT, or identity); distinct from RQ.encode \\
$g$ & Base model trained on $\mathcal{D}^{\circ}$ (frozen) \\
$\alpha$ & Blending weight (cross-validated); $\alpha{=}0$ recovers $g$ \\
$\varphi(x)$ & Feature map: $L_2(f(x))$ \\
$b$, $\ell$ & Centroids per RQ level; number of levels \\
$\mathbf{C}$\,;\; $c_m$ & Code matrix $\in [b]^{n \times \ell}$; code at level $m$ \\
$y^r_i$ & RQ target for point $i$; $y^r_i = y_i - g(x_i)$ \\
$\mathbf{T}_{m,j}$ & RQ target sum $\sum_{i:\,c_{i,m}=j} y^r_i$ per cell \\
$N_{k,m}$ & Co-occurrence count matrix $\in \mathbb{N}^{b \times b}$ \;($k < m$) \\
$\mathbf{S}_{m,j}$\,;\; $n_{m,j}$ & Materialized accumulator / count (level $m$, cell $j$) \\
\bottomrule
\end{tabular}
\end{small}
\end{table}

\subsection{Instances: Schema \& Content}
\label{sec:model-pair}
\label{sec:codebook_desc}

QSS maintains $p \geq 1$ independent \emph{instances}
(Figure~\ref{fig:pipeline}, top).
Each instance consists of a frozen \emph{schema} (base model~$g_i$
and RQ codebook, trained on a retained subset
$\mathcal{D}^{\circ}_i$ of size $n^{\circ} < n$) and mutable
\emph{content} (accumulators aggregating all of $\mathcal{D}$).
The content stores only sufficient statistics ($\mathbf{T}$,
$n$, $N$), not the data itself; at our operating point, all
accumulators occupy ${\leq}\,51$\,MB even for ImageNet.
The deletable portion is
$\mathcal{D}^{\bullet}_i = \mathcal{D} \setminus \mathcal{D}^{\circ}_i$.
Because every instance's accumulators aggregate \emph{all}
of~$\mathcal{D}$ (using $\mathcal{D}^{\circ}_i$ only to train
the schema), decommissioning an instance incurs zero accuracy
loss.

Within each instance, the base model~$g$ and the RQ codebook are
trained on $\mathcal{D}^{\circ}$; all of $\mathcal{D}$ is then
encoded through the frozen codebook and accumulated into
mutable sufficient statistics (the residual sums~$\mathbf{T}$,
counters~$n$, and co-occurrence counts~$N$) visible at the
right of the building panel.
Prediction (Figure~\ref{fig:pipeline}, bottom-left) sums
the cell means along the query's code path and adds the base
model's output:
\begin{align}
  \hat{y}(x) \;=\; g(x) \;+\; \alpha \cdot
    \textsc{RQA.Predict}(x)\,,
  \label{eq:predict}
\end{align}
where $\alpha$ is cross-validated on a held-out fold along
with other hyperparameters.
Unlearning (Figure~\ref{fig:pipeline}, bottom-right) decrements
the affected accumulators and re-materializes; the schema
(codebook, $g$) is untouched.
If the deleted point lies in $\mathcal{D}^{\circ}$, the pair
is invalidated and a replacement is needed.

\paragraph{Guarantee-specific constructions.}
The neighboring relations were defined in Section~\ref{sec:intro}; the
construction must match the requested relation. Under label deletion, the
codebook and assignments may still depend on~$x_j$, but the released predictor
must equal one whose supervised state was built on~$\mathcal{D}^{-y_j}$. Under
instance deletion, every model component must equal one built on
$\mathcal{D}^{-j}$.
Three modes trade off these guarantees against accuracy
(Table~\ref{tab:methods}):
\begin{itemize}\setlength\itemsep{1pt}
\item \textbf{QSS-E} (instance): The codebook is trained on
  $\mathcal{D}^{\circ}$ only. Exact label and instance
  deletion for points in $\mathcal{D}^{\bullet}$.
\item \textbf{QSS-L} (label): The codebook is trained on all
  $\mathcal{D}$.  This leads to better codebook resolution and
  improved model performance, but only allows for exact label
  deletion.
\item \textbf{QSS-DP}: QSS-L plus Gaussian noise on published cell
  means.  Provides $(\varepsilon,\delta)$-label differential privacy
  (Proposition~\ref{prop:label_dp}).
\end{itemize}

\subsection{RQA Data Structure}
\label{sec:data_structure_desc}

Residual Quantization (RQ)~\cite{chen2010approximate} encodes
each embedding as a sum of $\ell$ centroid vectors, one per level,
each chosen from $b$ candidates; level~$m$ quantizes the residual
from levels $1,\ldots,m{-}1$.
We treat the RQ codebook and its encoder as a frozen black box
(trained via Faiss~\cite{johnson2019billion}); the \emph{RQA}
(RQ-Accumulator) data structure defined below stores sufficient
statistics on top of these codes.
The codebook may be trained on $\mathcal{D}^{\circ}$ only (QSS-E,
exact instance and label deletion) or all of $\mathcal{D}$ (QSS-L,
better resolution, exact label deletion only;
Table~\ref{tab:methods}).

The data structure stores three sum-decomposable quantities per
code tuple (each a sum over points, so adding or removing a point
costs $O(\ell^2)$ decrements):
\begin{align}
  \mathbf{T}_{m,j} &= \textstyle\sum_{i:\,c_{i,m}=j} y^r_i
    & &\text{(target sums)}\,,
    \label{eq:target_sums} \\
  n_{m,j} &= |\{i : c_{i,m}{=}j\}|
    & &\text{(cell counts)}\,,
    \label{eq:cell_counts} \\
  N_{k,m}[j',j] &= |\{i : c_{i,k}{=}j',\;c_{i,m}{=}j\}|
    & &\text{(co-occurrence,\; $k < m$)}\,.
    \label{eq:cooccurrence}
\end{align}
Prediction recovers per-level cell means
$\boldsymbol{\mu}_{m,j} = \mathbf{S}_{m,j}/n_{m,j}$
via \textsc{RQA.Materialize} (Equation~\ref{eq:materialize}),
which subtracts earlier-level contributions using
co-occurrence counts, analogous to gradient
boosting~\cite{friedman2001greedy} but avoiding exponential
fragmentation across levels.
For query~$x$:
\begin{align}
  \textsc{RQA.Predict}(x) &= \textstyle\sum_{m=1}^{\ell}
    \boldsymbol{\mu}_{m,\,c_m(x)}\,.
  \label{eq:rq_predict}
\end{align}
Unlearning is exact: decrement and re-materialize
(Figure~\ref{fig:architecture}).
The complete RQA pseudocode appears in
Algorithm~\ref{alg:qss} (Appendix~\ref{app:rqa_pseudocode}).

\section{Analysis}
\label{sec:guarantees}
\label{sec:codebook_guarantees}

Table~\ref{tab:methods} summarizes the operating modes, their
complexity, and deletion guarantees
(proofs in Appendix~\ref{app:qss_complexity}).

\begin{table}[t]
\centering
\caption{QSS complexity and deletion guarantees.
QSS-E, QSS-L, and QSS-DP differ in codebook training set
and noise; all use the RQA data structure
(Algorithm~\ref{alg:qss}). When a deletion hits
$\mathcal{D}^{\circ}$ (probability $n^{\circ}/n$), $g$ and the
codebook must be rebuilt.}
\label{tab:methods}
\smallskip
\begin{small}
\begin{tabular}{lcccc}
\toprule
 & QSS-E & QSS-L & QSS-DP & SISA \\
\midrule
Codebook & $\mathcal{D}^{\circ}$ & All $\mathcal{D}$
  & All $\mathcal{D}$ & --- \\
\midrule
Train & \multicolumn{3}{c}{$t_g(n^{\circ}) + O(n\ell bd)$} & $O(n)$ \\
Predict & \multicolumn{3}{c}{$t_g + O(\ell bd)$} & $O(S{\cdot}dk)$ \\
\midrule
Unlearn from $\mathcal{D}^{\bullet}$ & \multicolumn{3}{c}{$O(\ell^2 b^2 k)$}
  & $O(n/S)$ \\
Unlearn from $\mathcal{D}^{\circ}$ & \multicolumn{3}{c}{$t_{\text{rebuild}}$}
  & $O(n/S)$ \\
$\mathbb{E}[\text{unlearn}]$
  & \multicolumn{3}{c}{$(1{-}\tfrac{n^{\circ}}{n})\,O(\ell^2 b^2 k) +
  \tfrac{n^{\circ}}{n}\,t_{\text{rebuild}}$}
  & $O(n/S)$ \\
\midrule
Label del.\ & Exact & Exact
  & $\varepsilon$-DP$^\dagger$ & Exact \\
Instance del.\ & Exact & ---
  & Empirical DP & Exact \\
\bottomrule
\end{tabular}
\end{small}

\smallskip
{\footnotesize
$t_g$: base-model cost (linear probe or MLP, model-dependent).
QSS-E: codebook on $\mathcal{D}^{\circ}$ only (instance guarantee).
QSS-L: codebook on all data (better accuracy).
QSS-DP: + Gaussian noise on cell means guaranteeing $(\varepsilon,\delta)$-label DP (Proposition~\ref{prop:label_dp}).
}
\end{table}

\subsection{Unlearning}

\begin{proposition}[Exact Instance Unlearning in QSS-E]
\label{prop:exact_unlearning}
\label{cor:label_unlearning}
In the QSS-E method $\hat{y}(x) = g(x) + \alpha \cdot
\textsc{RQA.Predict}(x)$, where the schema ($g$, $\varphi$,
and the codebook) is trained on the retained
portion~$\mathcal{D}^{\circ}$ and the RQA accumulators are
filled on all data, deleting any point $(x_j, y_j) \in
\mathcal{D}^{\bullet}$ via $\textsc{RQA.Unlearn}$ produces
predictions bitwise identical to retraining the full system
from scratch on $\mathcal{D}^{-j}$.
\end{proposition}

\begin{proof}
Since the schema ($\varphi$, codebook, $g$) is trained on
$\mathcal{D}^{\circ}$ and $(x_j, y_j) \in \mathcal{D}^{\bullet}$,
deleting point~$j$ changes neither code assignments nor residual
targets.  The accumulators $(\mathbf{T}, n, N)$ are additive, so
\textsc{RQA.Unlearn} (subtract $j$'s contributions) produces
exactly the values from \textsc{RQA.Build} on
$\mathcal{D} \setminus \{j\}$.
It remains to show that \textsc{RQA.Materialize} is a
deterministic function of $(\mathbf{T}, n, N)$.

By definition,
$\mathbf{S}_{m,j} = \sum_{i:\,c_{i,m}=j} r_m(x_i)$ where
$r_m(x_i) = y_i - \sum_{k=1}^{m-1}
\boldsymbol{\mu}_{k,c_{i,k}}$.
Distributing the outer sum and grouping points in cell~$j$ at
level~$m$ by their cell at level~$k$:
\begin{align}
  \mathbf{S}_{m,j}
  &= \mathbf{T}_{m,j}
    - \sum_{k=1}^{m-1} \sum_{j'=1}^{b}
      \underbrace{|\{i : c_{i,k}{=}j',\, c_{i,m}{=}j\}|}_{N_{k,m}[j',j]}
      \cdot \boldsymbol{\mu}_{k,j'},
      \label{eq:materialize}
\end{align}
with $\boldsymbol{\mu}_{m,j} = \mathbf{S}_{m,j} / n_{m,j}$
($\mathbf{0}$ if $n_{m,j} = 0$).
Levels are computed sequentially ($m = 1, 2, \ldots$) because
$\boldsymbol{\mu}_{k,j'}$ depends on accumulators at levels $< k$.
Since \textsc{RQA.Materialize} is thus deterministic in
$(\mathbf{T}, n, N)$, the predictions after unlearning are
bitwise identical to retraining from scratch on
$\mathcal{D} \setminus \{j\}$.
\end{proof}

Exact \emph{instance} deletion (no trace of the point in any
model component) requires the codebook to be independent of
deletable data.  Refitting $k$-means after a single-point
removal can flip every code
(Appendix~\ref{app:fragility},~\cite{muresanu2025fast}); QSS-E
avoids this by training the codebook on
$\mathcal{D}^{\circ}$~only.

\begin{remark}[Exact Label Unlearning in QSS-L]
\label{rem:qssl_label}
Proposition~\ref{prop:exact_unlearning} is stated for QSS-E
(codebook on $\mathcal{D}^{\circ}$), but exact \emph{label}
deletion also holds in QSS-L (codebook on all $\mathcal{D}$).
The codebook is trained by $k$-means on feature vectors
$\varphi(x_i)$ without access to labels, so it is
label-blind: labels enter only through the accumulators
$\mathbf{T}_{m,j}$ and $n_{m,j}$.
Subtracting a point's label contributions from these
accumulators and re-materializing produces predictions
identical to rebuilding the supervised accumulators from scratch on
$\mathcal{D}^{-y_j}$ with the same codebook.
What QSS-L does \emph{not} guarantee is instance deletion:
the codebook centroids depend on the deleted point's features,
so traces of the point's input remain in the schema
(Table~\ref{tab:methods}).
\end{remark}

\begin{proposition}[Label DP in QSS-DP]
\label{prop:label_dp}
Adding $\mathcal{N}(0, \sigma^2)$ noise to each published cell
mean provides $(\varepsilon, \delta)$-label differential privacy
with
\begin{align}
  \varepsilon = \frac{\|\Delta y^r\|^2}{2\,n_{\min}\,\sigma^2}
  + \sqrt{\frac{2\,\|\Delta y^r\|^2\,\ln(1/\delta)}
  {n_{\min}\,\sigma^2}},
\end{align}
where $\|\Delta y^r\|$ is the maximum residual-label change
from a single point's deletion and
$n_{\min} = \min_{m,j:\, n_{m,j}>0} n_{m,j}$ is the smallest
occupied cell count across all levels and centroids
(Equation~\ref{eq:cell_counts}; empty cells are excluded from
prediction). The bound is independent of both depth~$\ell$
and width~$b$.
\end{proposition}

\noindent Detailed proof in Appendix~\ref{app:proof_dp}.
Each centering step in \textsc{RQA.Materialize} is an
orthogonal projection, so the $L_2$ perturbation energy
from a single label change is non-increasing across levels.
Under zero-concentrated DP~\cite{bun2016concentrated}, the
per-level privacy costs $\rho_m = \|\boldsymbol{\delta}_m\|_2^2
/ (2\sigma^2)$ telescope:
$\sum_m \rho_m \leq \|\Delta y^r\|^2 / (2\,n_{\min}\,\sigma^2)$,
eliminating the depth factor entirely.
Note that QSS-DP provides \emph{label}-DP only: the noise masks
the influence of individual labels on cell means, but the codebook
and code assignments may still depend on the point's features
(unless QSS-E is used).
We perform empirical studies using membership inference
attacks to measure residual membership signal
(\S\ref{sec:experiments}).

\subsection{Operational guarantees}
\label{sec:amortized}
\label{prop:amortized}
\label{prop:ctmc}
\label{rem:adversarial}

When a deletion invalidates a instance, it begins rebuilding
immediately while the remaining instances continue serving.

\begin{proposition}[Rebuild Scheduling]
\label{prop:rebuild}
Under immediate rebuilds with $p{-}1$~instances serving during
each rebuild window~$T_r$, and deletion requests arriving as a
Poisson process with rate~$\lambda_d$:
\begin{enumerate}\setlength\itemsep{2pt}
\item Outage probability per window:
  $P_{\textup{out}} \approx
  (1 - e^{-\rho\lambda_d T_r})^{p-1}$,
  exponentially small in~$p$.
\item Rebuild overhead: at most $p\,H_p$ times the offline
  optimal, where $H_p = \sum_{i=1}^{p} 1/i$
  ($5.5{\times}$ for $p{=}3$).
\end{enumerate}
\end{proposition}

\noindent Proofs in
Appendices~\ref{app:qss_complexity}
and~\ref{app:rebuild_proofs}.
Under the eager policy, the expected inter-rebuild gap is
$1/(p\rho)$ deletions, so the expected cost per deletion is
$O(\ell^2 b^2 k)$ (fast-path re-materialization) plus
$p\rho\,t_{\textup{rebuild}}$ (amortized rebuild).
The offline optimal achieves gap $H_p/\rho$
(coupon-collector lower bound), giving the $p\,H_p$
competitive ratio.
Given that $\mathcal{D}^{\circ}$ membership is secret, an
adversary controls only the \emph{rate} of deletions, not which
instances are invalidated; a batch of $B$ simultaneous deletions
invalidates all $p$~instances with probability
$(1{-}(1{-}\rho)^B)^p$ ($B > 1{,}300$ to exceed~$1\%$
at $\rho{=}0.5\%$, $p{=}3$).
Steady-state cost rates for deployment sizing appear
in Appendix~\ref{app:rebuild_proofs}.

\subsection{Prediction quality}

\begin{proposition}[Variance Decay]
\label{prop:variance_decay}
Within the RQ cascade, the residual variance after $m$ levels
satisfies
$\mathrm{Var}(r_m) = \mathrm{Var}(y^r)\,\prod_{j=1}^{m}(1 - \eta_j)$,
where $\eta_m = \mathrm{Var}(\mathbb{E}[r_{m-1} \mid c_m]) /
\mathrm{Var}(r_{m-1})$ is the variance explained at level~$m$.
If $\eta_m \geq \gamma > 0$, MSE decays as $(1{-}\gamma)^\ell$.
\end{proposition}

\noindent Proof in Appendix~\ref{app:proof_variance}.
The cascade is analogous to gradient
boosting~\cite{friedman2001greedy} with quantization centroids as
weak learners. Excess risk scales as $b/n$, not $b^\ell/n$: each
level estimates cell means independently with $\sim n/b$ points
(Appendix~\ref{app:excess_risk}).

\medskip

\begin{remark}[Depth Collapse]
\label{rem:depth_collapse}
Because residual norms shrink geometrically
(Proposition~\ref{prop:variance_decay}), deeper levels add
estimation noise without reducing approximation error.  The
practical rule is \emph{width $>$ depth}
(Appendix~\ref{app:excess_risk}).
\end{remark}

\section{Experimental Results}
\label{sec:experiments}

\subsection{Setup}
\label{sec:exp_setup}

We evaluate on 15 datasets with $n \geq 50$K (where deletion speed
is a genuine bottleneck), spanning three modalities and $k$ from~2
to~1{,}000 (Table~\ref{tab:summary}).
A frozen encoder maps inputs to $L_2$-normalized embeddings:
CLIP ViT-L/14 (768-d) for 5 vision tasks,
all-MiniLM-L6-v2 (384-d) for 4 text tasks, and
LightGBM leaf contributions (trained on $\mathcal{D}^{\circ}$) for
6 tabular tasks.
All 27 viable datasets are reported in
Appendix~\ref{app:additional_experiments}; no dataset with
$n \geq 50$K was excluded based on QSS performance.

\paragraph{Data splits and operating point.}
We set $\rho = n^{\circ}/n = 0.5\%$ throughout, so that 0.5\% of
the data is retained and 99.5\% is deletable.  This gives
$n^{\circ}$ ranging from 256 (SST-2) to 41K (HIGGS).
A base model $g$ (linear probe when $n^{\circ} < 5{,}000$;
one-hidden-layer MLP otherwise) is trained on $\mathcal{D}^{\circ}$
and frozen.  QSS corrects $g$'s residuals using accumulators built
on all of $\mathcal{D}$.  The blending weight $\alpha$ and codebook
hyperparameters $(b, \ell)$ are cross-validated on
$\mathcal{D}^{\circ}$; $\alpha = 0$ recovers $g$ exactly,
guaranteeing QSS~$\geq$~$\mathcal{D}^{\circ}$-only.

\paragraph{Baselines.}
We compare three baselines.
\emph{$\mathcal{D}^{\circ}$-only} (floor) trains the base model~$g$
on $\mathcal{D}^{\circ}$ alone; QSS is guaranteed to match or exceed
this by construction ($\alpha{=}0$ fallback).
\emph{SISA}~\cite{bourtoule2021machine} (ceiling) partitions
$\mathcal{D}^{\bullet}$ into $S{=}5$ shards, trains the same model
architecture as~$g$ on $\mathcal{D}^{\circ} \cup \text{shard}_k$, and
majority-votes predictions; unlearning retrains the affected shard at
cost $O(n/S)$. We use $S{=}5$
following~\cite{bourtoule2021machine}, the most favorable accuracy
setting for SISA since larger~$S$ reduces per-shard accuracy.
\emph{Exact ridge}~\cite{quan2026exact} fits a ridge head using all training labels and maintains
$X^\top X$ and $X^\top Y$ for rank-one exact deletion. We report matched
submitted-split comparisons on AG~News and SST-2, the two tasks for which the
submitted frozen-SBERT pipeline can be reconstructed exactly
(Appendix~\ref{app:analytic}).

Table~\ref{tab:baselines} places QSS among existing
methods: QSS is exact and $n$-independent, unlike SISA, and
avoids the full LLM forward pass
of~\cite{muresanu2025fast}.

\begin{table}[t]
\centering
\caption{Architectural comparison with unlearning methods.}
\label{tab:baselines}
\smallskip
\begin{small}
\begin{tabular}{lccccc}
\toprule
Method & Exact? & Unlearn & Predict & Indep.\ of $n$? \\
\midrule
Retrain & Yes & $O(n)$ & 0 & No \\
SISA~\cite{bourtoule2021machine} & Yes & $O(n/S)$ & $\times S$ & No \\
Certified~\cite{guo2020certified} & $\varepsilon$ & $O(d^3)$ & 0 & Yes \\
SLUG~\cite{cai2024slug} & No & $O(d^2)$ & 0 & Yes \\
\cite{muresanu2025fast} & Yes & $O(1)$ & LLM fwd & Yes \\
Ridge/ACU~\cite{tang2025acu,quan2026exact} & Yes & $O(d^2{+}dk)$ & $O(dk)$ & Yes \\
\midrule
QSS (ours) & Yes & $O(\ell^2 b^2 k)$ & $O(\ell bd)$ & Yes \\
\bottomrule
\end{tabular}
\end{small}
\end{table}

\subsection{The Accuracy--Speed Tradeoff}
\label{sec:pareto}

The retention ratio $\rho = n^{\circ}/n$ controls the
accuracy--speed tradeoff. A larger $\rho$ yields a stronger base
model $g$ and a higher-fidelity codebook, but increases the
rebuild probability. The expected per-deletion cost is
\begin{equation}
\mathbb{E}[t] \;=\; (1{-}\rho)\,t_{\text{RQ}} \;+\;
\rho\,t_{\text{rebuild}}\,,
\label{eq:pareto}
\end{equation}
where $t_{\text{RQ}} = O(\ell^2 b^2 k)$ is the fast-path
accumulator update (independent of $n$) and
$t_{\text{rebuild}}$ is the full pipeline cost (retrain $g$,
re-extract embeddings, retrain codebook, re-encode and
re-accumulate).
At $\rho{=}0.5\%$, the $\rho\,t_{\text{rebuild}}$ term is
small even when $t_{\text{rebuild}}$ itself is large: for
HIGGS ($n{=}11$M), the expected cost is ${\approx}\,18$\,s
vs.\ SISA's $114$\,s (Table~\ref{tab:timing}).
Figure~\ref{fig:pareto_d0} traces the tradeoff as $\rho$
varies from $0.5\%$ to $5\%$; most low-$k$ datasets cluster
near the break-even line, while high-$k$ outliers narrow
their gap as $\rho$ increases.

\begin{figure}[t]
\centering
\includegraphics[width=0.75\linewidth]{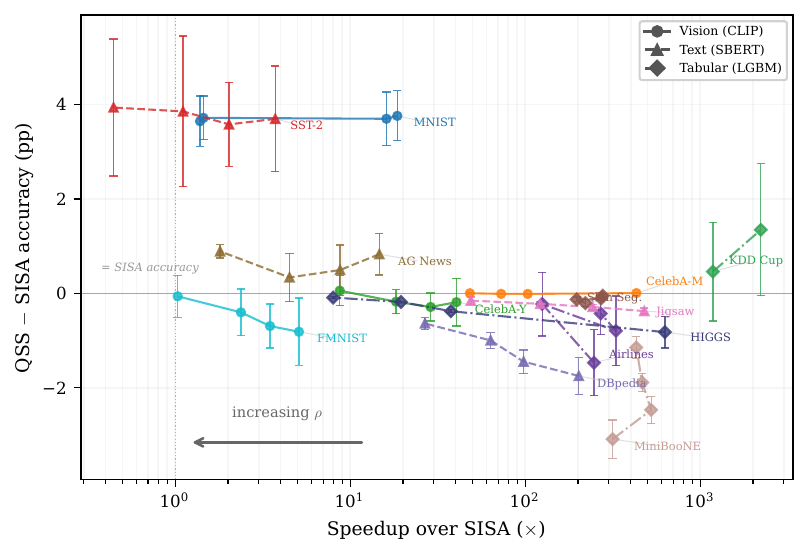}
\caption{QSS-L accuracy--speedup Pareto as $\rho$ varies from
$0.5\%$ to $5\%$ on the 13 low-$k$ datasets. Both axes are
relative to SISA ($x{=}1$, $y{=}0$). Each trajectory traces
one dataset; moving \emph{left} corresponds to increasing $\rho$
(lower speedup, higher accuracy).
Most trajectories cluster near the break-even line, trading
modest accuracy for large speedups.
Full version with all 15 datasets in
Figure~\ref{fig:pareto_d0_full} (appendix).}
\label{fig:pareto_d0}
\end{figure}
\subsection{Accuracy and Timing}
\label{sec:quality}

\begin{table}[t]
\centering
\caption{Guarantee-aligned accuracy at
$\rho{=}0.5\%$ on low-$k$ datasets (mean $\pm$ std over
5 random splits; $^\ddagger$2--3 seeds).
$g$-only: base model on $\mathcal{D}^{\circ}$ (floor).
QSS-L: one fixed label-deletion construction per feature pipeline; modes are not
selected using test accuracy. QSS-E is reported separately in
Appendix~\ref{app:mode_separation}.
SISA: $S{=}5$ shard retraining~\cite{bourtoule2021machine}.
$\Delta$: QSS-L$\,-\,$SISA accuracy (pp).
High-$k$ and QSS-E results appear in
Table~\ref{tab:fixed_modes} (appendix).}
\label{tab:summary}
\label{tab:timing_main}
\smallskip
\begin{small}
\setlength{\tabcolsep}{2pt}
\begin{tabular}{llrcccc}
\toprule
Task & $~k$ & $n~~$
  & $g$-only & QSS-L & SISA
  & $\Delta$ \\
\midrule
\multicolumn{7}{l}{\emph{Vision (CLIP ViT-L/14)}} \\
MNIST         & ~10 &   53K~~ & .814{\tiny$\pm$.009} & .944{\tiny$\pm$.003} & .921{\tiny$\pm$.001} & $+$2.4 \\
FMNIST        & ~10 &   53K~~ & .749{\tiny$\pm$.015} & .825{\tiny$\pm$.014} & .854{\tiny$\pm$.003} & $-$2.9 \\
CelebA Young  & ~2  &  137K~~ & .886{\tiny$\pm$.002} & .893{\tiny$\pm$.005} & .902{\tiny$\pm$.001} & $-$0.9 \\
CelebA Male   & ~2  &  137K~~ & .992{\tiny$\pm$.000} & .993{\tiny$\pm$.001} & .993{\tiny$\pm$.001} & $+$0.0 \\
\midrule
\multicolumn{7}{l}{\emph{Text (all-MiniLM-L6-v2)}} \\
SST-2         & ~2  &   51K~~ & .785{\tiny$\pm$.008} & .854{\tiny$\pm$.020} & .835{\tiny$\pm$.003} & $+$1.8 \\
AG News       & ~4  &   96K~~ & .861{\tiny$\pm$.003} & .899{\tiny$\pm$.004} & .895{\tiny$\pm$.001} & $+$0.4 \\
DBpedia       & ~14 &  473K~~ & .952{\tiny$\pm$.002} & .958{\tiny$\pm$.004} & .977{\tiny$\pm$.001} & $-$1.9 \\
Jigsaw        & ~2  & 1.4M~~  & .923{\tiny$\pm$.000} & .924{\tiny$\pm$.001} & .927{\tiny$\pm$.000} & $-$0.4 \\
\midrule
\multicolumn{7}{l}{\emph{Tabular (LightGBM leaf contributions)}} \\
Skin Seg.     & ~2  &  245K~~ & .986{\tiny$\pm$.002} & .997{\tiny$\pm$.001} & .999{\tiny$\pm$.000} & $-$0.2 \\
Airlines      & ~2  &  500K~~ & .627{\tiny$\pm$.004} & .640{\tiny$\pm$.002} & .655{\tiny$\pm$.007} & $-$1.5 \\
MiniBooNE     & ~2  &  130K~~ & .903{\tiny$\pm$.004} & .912{\tiny$\pm$.004} & .943{\tiny$\pm$.002} & $-$3.1 \\
HIGGS$^\ddagger$ & ~2 & 11M~~ & .728{\tiny$\pm$.000} & .730{\tiny$\pm$.001} & .740{\tiny$\pm$.003} & $-$1.0 \\
KDD Cup       & ~17 &  494K~~ & .996{\tiny$\pm$.001} & .999{\tiny$\pm$.000} & .994{\tiny$\pm$.011} & $+$0.5 \\
\bottomrule
\end{tabular}
\end{small}
\end{table}

Table~\ref{tab:summary} reports accuracy and expected deletion
latency at $\rho{=}0.5\%$ on 13 low-$k$ datasets (mean over
5 random splits; full results including high-$k$ tasks in
Table~\ref{tab:fixed_modes}, appendix).
QSS matches or exceeds $\mathcal{D}^{\circ}$-only on all
datasets, confirming that cross-validated $\alpha$ reliably
extracts value from deletable data without risk.

Fixed QSS-L exceeds SISA accuracy on 4 of 13 datasets
(MNIST, SST-2, AG~News, and KDD~Cup) and trails by
${\leq}\,2$\,pp on 7~more.
The separately measured QSS-L operating points achieve 4--483$\times$ lower
expected deletion cost (Figure~\ref{fig:pareto_d0} and
Table~\ref{tab:timing}); every latency includes the rare rebuild event.

\paragraph{Deletion timing.}
The fast path dominates: at $\rho{=}0.5\%$, the
$\rho\,t_{\text{rebuild}}$ term contributes $<1$\,s for most
datasets.
On Jigsaw ($n{=}1.4$M), QSS-L deletes in 0.2\,s expected vs.\
SISA's 109\,s ($483\times$).
The exceptions are HIGGS and KDD~Cup, where large
$t_{\text{rebuild}}$ limits the speedup to~$6\times$
(Table~\ref{tab:timing}, appendix).

Aggregate analysis (Figure~\ref{fig:aggregate}, appendix)
confirms that $k$ is the primary factor: the gap clusters near
zero for low-$k$ (median $-0.4$\,pp) but separates at
$k{=}1{,}000$ (\S\ref{sec:limitations}).

\subsection{Correctness and Privacy}
\label{sec:unlearn_correct}
\label{sec:privacy}

We verify the additive exactness mechanism empirically by deleting 100 points
and comparing accumulator-subtracted predictions to rebuilt statistics. Across
nine datasets spanning all three modalities, \emph{all differences are exactly
0.0}: bitwise identical for the relevant fixed-schema relation. A targeted
QSS-L audit on AG~News obtains the same result for batches of 1, 10, 100, and
1,000 label deletions. Deleting up to 50\% of ImageNet (640K points) likewise
produces bitwise-identical predictions, with only 0.25\,pp accuracy loss from
the reduced labels (Table~\ref{tab:retention}, appendix).

\begin{figure}[t!]
\centering
\begin{subfigure}[t]{0.52\linewidth}
  \centering
  \includegraphics[width=\linewidth]{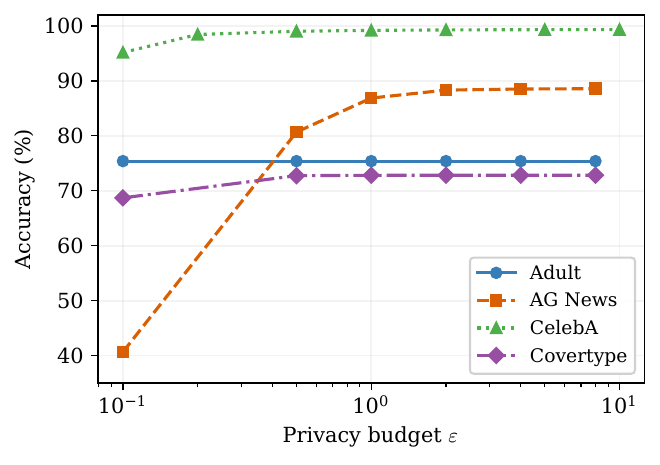}
  \caption{$\varepsilon$-label-DP accuracy.
  Adult (binary) is insensitive; Covertype (7-class) and
  AG~News (4-class) degrade gracefully.}
  \label{fig:epsilon_accuracy}
\end{subfigure}\hfill
\begin{subfigure}[t]{0.42\linewidth}
  \centering
  \includegraphics[width=\linewidth]{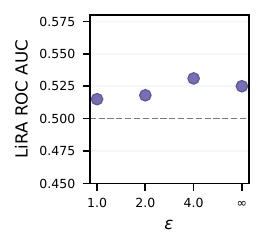}
  \caption{LiRA membership inference on CelebA.
  AUC~$\approx 0.52$ at all $\varepsilon$, including
  $\varepsilon{=}\infty$.}
  \label{fig:mia}
\end{subfigure}
\caption{Privacy validation.
\textbf{(a)}~DP noise preserves most accuracy at
$\varepsilon \geq 1$.
\textbf{(b)}~Even without DP, QSS leaks no membership
information (64 shadow models, 1000 audit points).}
\label{fig:privacy_mechanism}
\end{figure}

QSS predictions depend on cell-level aggregates ($n_{m,j} \gg 1$), not individual points, limiting membership leakage by construction (Figure~\ref{fig:privacy_mechanism}). Adding $(\varepsilon,\delta)$-label DP (Proposition~\ref{prop:label_dp}) preserves most accuracy at $\varepsilon \geq 1$ (panel~a). A LiRA attack~\cite{carlini2022membership} (64 shadow models, 1000 audit points) achieves AUC~$\approx 0.52$, even at $\varepsilon{=}\infty$ (panel~b; Table~\ref{tab:lira}, appendix).

\subsection{Discussion and Limitations}
\label{sec:limitations}

\paragraph{High-$k$ gap.}
On ImageNet ($k{=}1{,}000$, $-15.6$\,pp) and Covertype
($k{=}7$, $-10.7$\,pp), QSS lags SISA at $\rho{=}0.5\%$.
The codebook partitions embedding space without regard to
class boundaries, diluting cell-mean signal when $k$ is large.
Level-0 cell purity predicts this: MNIST achieves 96.6\%
purity ($+3.8$\,pp over SISA), Covertype only 63.1\%
(Figure~\ref{fig:umap_purity}, appendix).
The gap narrows with $\rho$: Covertype shrinks to $-5$\,pp
at $\rho{=}5\%$ (Table~\ref{tab:timing}).

\paragraph{Rebuild cost and practical guidance.}
The $\rho\,t_{\text{rebuild}}$ term scales with $n$; for HIGGS
($n{=}11$M), this limits the speedup to~$67\times$.
On tabular datasets, accumulator re-fill converges early
(1\% of~$n$ for Skin Segmentation, 10\% for HIGGS;
Table~\ref{tab:convergence}, appendix).
At $\rho{=}0.5\%$ with $k{\leq}17$, QSS trails SISA by at
most 3\,pp while deleting 4--483$\times$ faster; SISA remains
preferable when $k$ is large.

\paragraph{Schema capacity.}
The submitted linear and one-hidden-layer heads are operating points, not method
restrictions. A richer schema may improve difficult tasks, but a larger retained
set raises both rebuild probability and rebuild cost. On Covertype, a three-layer
head gains little over linear at $\rho=0.5\%$ but gains $10.74$ pp at
$\rho=5\%$, while its training time rises from $1.29$ to $14.65$ seconds.
Appendix~\ref{app:capacity} reports the joint capacity--schema sweep and shows
that the benefit is dataset-dependent.

\subsection{Summary}
\label{sec:conclusion}

QSS separates \emph{schema} from \emph{content}, reducing ordinary deletion to constant-time arithmetic independent of~$n$; redundant schemas drive rebuild probability to~$\rho^p$ with zero accuracy loss. QSS-L provides exact label deletion with the competitive accuracy reported above, while QSS-E exposes the additional cost of exact instance deletion. Unlike SISA's $S$ pipelines, QSS is a single model with a sidecar data structure.


\begin{ack}
This work received no external funding. The authors declare that they have no
competing interests.
\end{ack}

\bibliographystyle{plain}
\bibliography{references}

@article{tang2025acu,
  title={Towards Efficient and Exact Forgetting Services in Pre-Trained-Model-based Continual Learning},
  author={Huang, Yajiang and Tang, Jianheng and Fan, Kejia and Zhuang, Huiping and Liu, Anfeng and Wang, Tian and Liu, Yunhuai and Dong, Mianxiong and Song, Houbing Herbert},
  journal={arXiv preprint arXiv:2505.12239},
  year={2025}
}

@article{quan2026exact,
  title={Exact Federated Continual Unlearning for Ridge Heads on Frozen Foundation Models},
  author={Quan, Yijun and Wu, Wentai and Montana, Giovanni},
  journal={arXiv preprint arXiv:2603.12977},
  year={2026}
}

@inproceedings{ke2017lightgbm,
  title={LightGBM: A Highly Efficient Gradient Boosting Decision Tree},
  author={Ke, Guolin and Meng, Qi and Finley, Thomas and Wang, Taifeng and Chen, Wei and Ma, Weidong and Ye, Qiwei and Liu, Tie-Yan},
  booktitle={Advances in Neural Information Processing Systems},
  volume={30},
  pages={3146--3154},
  year={2017},
  eprint={1711.05101},
  archiveprefix={arXiv}
}

@inproceedings{brophy2021machine,
  title={Machine Unlearning for Random Forests},
  author={Brophy, Jonathan and Lowd, Daniel},
  booktitle={International Conference on Machine Learning},
  pages={1092--1104},
  year={2021},
  organization={PMLR},
  eprint={2009.05567},
  archiveprefix={arXiv}
}

@inproceedings{ge2013optimized,
  title={Optimized Product Quantization},
  author={Ge, Tiezheng and He, Kaiming and Ke, Qifa and Sun, Jian},
  booktitle={IEEE Transactions on Pattern Analysis and Machine Intelligence},
  volume={36},
  number={4},
  pages={744--755},
  year={2014},
  doi={10.1109/TPAMI.2013.240}
}

@article{friedman2001greedy,
  title={Greedy Function Approximation: A Gradient Boosting Machine},
  author={Friedman, Jerome H.},
  journal={Annals of Statistics},
  volume={29},
  number={5},
  pages={1189--1232},
  year={2001},
  doi={10.1214/aos/1013203451}
}

@inproceedings{bourtoule2021machine,
  title={Machine Unlearning},
  author={Bourtoule, Lucas and Chandrasekaran, Varun and Choquette-Choo, Christopher A and Jia, Hengrui and Travers, Adelin and Zhang, Baoyuan and Lie, David and Papernot, Nicolas},
  booktitle={IEEE Symposium on Security and Privacy (SP)},
  pages={141--159},
  year={2021},
  doi={10.1109/SP40001.2021.00019}
}

@inproceedings{golatkar2020eternal,
  title={Eternal Sunshine of the Spotless Net: Selective Forgetting in Deep Networks},
  author={Golatkar, Aditya and Achille, Alessandro and Soatto, Stefano},
  booktitle={IEEE/CVF Conference on Computer Vision and Pattern Recognition (CVPR)},
  pages={9301--9309},
  year={2020},
  doi={10.1109/CVPR42600.2020.00932}
}

@article{johnson2019billion,
  title={Billion-scale similarity search with {GPU}s},
  author={Johnson, Jeff and Douze, Matthijs and Jégou, Hervé},
  journal={IEEE Transactions on Big Data},
  volume={7},
  number={3},
  pages={535--547},
  year={2021},
  doi={10.1109/TBDATA.2019.2921572}
}

@inproceedings{carlini2022membership,
  title={Membership Inference Attacks From First Principles},
  author={Carlini, Nicholas and Chien, Steve and Nasr, Milad and Song, Shuang and Terzis, Andreas and Tramer, Florian},
  booktitle={IEEE Symposium on Security and Privacy},
  year={2022}
}

@article{chen2010approximate,
  title={Approximate Nearest Neighbor Search by Residual Vector Quantization},
  author={Chen, Yongjian and Guan, Tao and Wang, Cheng},
  journal={Sensors},
  volume={10},
  number={12},
  pages={11259--11273},
  year={2010},
  doi={10.3390/s101211259}
}

@inproceedings{ginart2019making,
  title={Making {AI} Forget You: Data Deletion in Machine Learning},
  author={Ginart, Antonio and Guan, Melody Y. and Valiant, Gregory and Zou, James},
  booktitle={Advances in Neural Information Processing Systems (NeurIPS)},
  volume={32},
  pages={3513--3526},
  year={2019},
  eprint={1907.05012},
  archiveprefix={arXiv}
}

@article{nguyen2022survey,
  title={A Survey of Machine Unlearning},
  author={Nguyen, Thanh Tam and Huynh, Thanh Trung and Ren, Zhao and Nguyen, Phi Le and Liew, Alan Wee-Chung and Yin, Hongzhi and Nguyen, Quoc Viet Hung},
  journal={ACM Transactions on Intelligent Systems and Technology},
  volume={16},
  number={5},
  pages={108:1--108:46},
  year={2025},
  doi={10.1145/3749987}
}

@inproceedings{guo2020certified,
  title={Certified Data Removal from Machine Learning Models},
  author={Guo, Chuan and Goldstein, Tom and Hannun, Awni Y. and van der Maaten, Laurens},
  booktitle={International Conference on Machine Learning (ICML)},
  pages={3832--3842},
  year={2020},
  organization={PMLR}
}

@inproceedings{koh2017understanding,
  title={Understanding Black-box Predictions via Influence Functions},
  author={Koh, Pang Wei and Liang, Percy},
  booktitle={ICML},
  year={2017}
}

@inproceedings{cai2024slug,
  title={Targeted Unlearning with Single Layer Unlearning Gradient},
  author={Cai, Zikui and Tan, Yaoteng and Asif, M. Salman},
  booktitle={International Conference on Machine Learning (ICML)},
  year={2025},
  note={arXiv:2407.11867}
}

@article{chen2025provable,
  title={Provable Unlearning with Gradient Ascent on Two-Layer {ReLU} Neural Networks},
  author={Melamed, Odelia and Yehudai, Gilad and Vardi, Gal},
  journal={arXiv preprint arXiv:2510.14844},
  year={2025}
}

@article{mishra2026quail,
  title={{QUAIL}: Quantization Aware Unlearning for Mitigating Misinformation in {LLM}s},
  author={Mishra, Himanshu and Mehreen, Kanwal},
  journal={arXiv preprint arXiv:2601.15538},
  year={2026}
}

@article{lunar2025,
  title={{LLM} Unlearning via Neural Activation Redirection},
  author={Shen, William F. and Qiu, Xinchi and Kurmanji, Meghdad and Iacob, Alex and Sani, Lorenzo and Chen, Yihong and Cancedda, Nicola and Lane, Nicholas D.},
  journal={arXiv preprint arXiv:2502.07218},
  year={2025}
}

@article{nguyen2026oeu,
  title={Forget by Uncertainty: Orthogonal Entropy Unlearning for Quantized Neural Networks},
  author={Zhang, Tian and Tong, Yujia and Dong, Junhao and Xu, Ke and Wang, Yuze and Yuan, Jingling},
  journal={arXiv preprint arXiv:2602.00567},
  year={2026}
}

@inproceedings{muresanu2025fast,
  title={Fast Exact Unlearning for In-Context Learning Data for {LLM}s},
  author={Muresanu, Andrei Ioan and Thudi, Anvith and Zhang, Michael R. and Papernot, Nicolas},
  booktitle={ICML},
  year={2025},
  note={arXiv:2402.00751}
}

@article{cheng2024mubench,
  title={{MU-Bench}: A Multitask Multimodal Benchmark for Machine Unlearning},
  author={Cheng, Jiali and Amiri, Hadi},
  journal={arXiv preprint arXiv:2406.14796},
  year={2024}
}

@article{mantelero2013eu,
  title={The eu proposal for a general data protection regulation and the roots of the 'right to be forgotten'},
  author={Mantelero, Alessandro},
  journal={Computer Law \& Security Review},
  volume={29},
  number={3},
  pages={229--235},
  year={2013},
  publisher={Elsevier}
}

@misc{ccpa_unlearning,
  author       = {{California Department of Justice}},
  title        = {California Consumer Privacy Act ({CCPA})},
  year         = {2024},
  url          = {https://oag.ca.gov/privacy/ccpa},
  note         = {Accessed: 2024-05-22},
  organization = {Office of the Attorney General}
}

@article{schoepf2024potion,
  title={Potion: Towards Poison Unlearning},
  author={Schoepf, Stefan and Foster, Jack and Brintrup, Alexandra},
  journal={Journal of Data-centric Machine Learning Research},
  year={2024},
  eprint={2406.09173},
  archiveprefix={arXiv}
}

@inproceedings{liu2025mitigating,
  title={Mitigating biases in language models via bias unlearning},
  author={Liu, Dianqing and Liu, Yi and Jin, Guoqing and Mao, Zhendong},
  booktitle={Proceedings of the 2025 Conference on Empirical Methods in Natural Language Processing},
  pages={4160--4178},
  year={2025}
}

@inproceedings{bun2016concentrated,
  title={Concentrated Differential Privacy: Simplifications, Extensions, and Lower Bounds},
  author={Bun, Mark and Steinke, Thomas},
  booktitle={Theory of Cryptography},
  pages={635--658},
  year={2016},
  doi={10.1007/978-3-662-53641-4_24},
  url={https://arxiv.org/abs/1605.02065}
}

\appendix
\section{Proofs}
\label{app:proofs}

\subsection{Proof of Proposition~\ref{prop:exact_unlearning}
  (Exact Label Unlearning)}
\label{app:proof_unlearning}

The proof appears in the main text (\S\ref{sec:guarantees}).
We expand Step~2 here with explicit subtraction formulas.

Let $(x_j, y_j) \in \mathcal{D}^{\bullet}$.
\textsc{RQA.Unlearn} subtracts point~$j$'s contributions:
\begin{align}
  \mathbf{T}'_{m,c} &= \mathbf{T}_{m,c} - y^r_j\,\mathbf{1}[c_{j,m}=c],
  &\quad n'_{m,c} &= n_{m,c} - \mathbf{1}[c_{j,m}=c], \nonumber \\
  N'_{k,m}[c',c] &= N_{k,m}[c',c] - \mathbf{1}[c_{j,k}=c',\, c_{j,m}=c].
\end{align}
These are exactly the values obtained by running
\textsc{RQA.Build} on $\mathcal{D} \setminus \{j\}$.

\subsection{Materialization Lemma}
\label{app:materialization}

The materialization recurrence (Equation~\ref{eq:materialize})
and its derivation appear in the proof of
Proposition~\ref{prop:exact_unlearning} (\S\ref{sec:guarantees}).

\paragraph{Complexity.}
Each $(k,m)$ pair requires a matrix--matrix product $N_{k,m}^\top
\boldsymbol{\mu}_k$ of dimensions $b \times b$ times $b \times k$, costing
$O(b^2 k)$. Summing over all $\binom{\ell}{2}$ pairs gives
$O(\ell^2 b^2 k)$ total. The computation has $\ell$ sequential rounds
(level~$m$ depends on $\boldsymbol{\mu}_{1}, \ldots,
\boldsymbol{\mu}_{m-1}$); within each round, the $m{-}1$ matmuls are
independent and parallelizable.

\paragraph{Storage.}
The stored state is $\binom{\ell}{2} b^2$ integers (co-occurrence
counts) plus $\ell b k$ floats (raw label sums) plus $\ell b$ integers (cell
counts). The materialized cache adds $\ell b k$ floats (residual
accumulators for prediction). At the operating points used in our
experiments ($b{\leq}256$, $\ell{\leq}8$): Adult~1\,MB,
Covertype~3.5\,MB, AG~News~5.1\,MB, CelebA~20\,MB,
ImageNet~51\,MB. Co-occurrence counts dominate for $k \ll b$;
label sums dominate for large~$k$ (ImageNet: 31\,MB of 51\,MB).
This storage trades static memory for eliminating $O(n)$ retraining.

\paragraph{Batch unlearning.}
Deleting $B$ points requires $O(B\ell^2)$ count decrements followed by a
single $O(\ell^2 b^2 k)$ re-materialization. The per-deletion amortized
cost is $O(\ell^2 + \ell^2 b^2 k / B)$.

\subsection{Proof of Proposition~\ref{prop:variance_decay}
  (Variance Decay)}
\label{app:proof_variance}

\begin{proof}
By the law of total variance:
\begin{align}
  \mathrm{Var}(r_{m-1})
  &= \mathrm{Var}\bigl(\mathbb{E}[r_{m-1} \mid c_m]\bigr)
   + \mathbb{E}\bigl[\mathrm{Var}(r_{m-1} \mid c_m)\bigr].
\end{align}
Since $r_m = r_{m-1} - \mathbb{E}[r_{m-1} \mid c_m]$ and the conditional
expectation is the minimum-variance predictor,
$\mathbb{E}[r_m \mid c_m] = 0$. Therefore:
\begin{align}
  \mathrm{Var}(r_m)
  &= \mathbb{E}\bigl[\mathrm{Var}(r_{m-1} \mid c_m)\bigr] \nonumber \\
  &= \mathrm{Var}(r_{m-1})
     - \mathrm{Var}\bigl(\mathbb{E}[r_{m-1} \mid c_m]\bigr) \nonumber \\
  &= \mathrm{Var}(r_{m-1})(1 - \eta_m).
\end{align}
Unrolling the recursion from $m = 1$ to $m = \ell$ with $r_0 = y$:
\begin{align}
  \mathrm{Var}(r_\ell)
  = \mathrm{Var}(y) \prod_{j=1}^{\ell}(1 - \eta_j).
\end{align}
\end{proof}

\begin{remark}
Because the materialized views $\mathbf{S}_{m,j}$ are algebraically
equivalent to sequentially accumulated residuals
(Equation~\ref{eq:materialize}), the variance decay properties of the
cascade remain identical under the joint-count formulation.
\end{remark}

\subsection{Label Differential Privacy Analysis}
\label{app:proof_dp}

\noindent\textbf{Scope.}
This bound applies to \emph{label} DP: it bounds the effect of changing
a single label on the published cell means. The codebook is trained via
$k$-means on feature vectors without access to labels, so it is
label-blind. This mechanism does \emph{not} provide instance DP for the
data vectors themselves.

\begin{proof}[Proof of Proposition~\ref{prop:label_dp}]
Under the joint-count formulation, the materialized accumulator at
level~$m$ depends on all earlier cell means:
$S_{m,j} = T_{m,j} - \sum_{k<m}\sum_{j'} N_{k,m}[j',j]\,\mu_{k,j'}$.
When a single label $y_i$ changes by $\Delta y$, define the perturbation
vector $\mathbf{x}_m \in \mathbb{R}^N$ where $(\mathbf{x}_m)_i$ is the
change in point~$i$'s level-$m$ residual.
Each centering step in \textsc{RQA.Materialize} is an orthogonal
projection (subtract the cell mean from each entry), so
$\|\mathbf{x}_m\|_2 \leq \|\mathbf{x}_0\|_2 = \|\Delta y\|$.

Define the per-level $L_2$ sensitivity
$\|\boldsymbol{\delta}_m\|_2^2 = \sum_j \delta_{m,j}^2$.
By Cauchy--Schwarz on each cell,
$\delta_{m,j}^2 \leq \|(\mathbf{x}_{m-1})_{[\text{cell}\,j]}\|_2^2
/ n_{m,j}$.
Summing over cells at level~$m$:
\begin{align}
  \|\boldsymbol{\delta}_m\|_2^2
  \leq \frac{1}{n_{\min}}
  \sum_j \|(\mathbf{x}_{m-1})_{[\text{cell}\,j]}\|_2^2
  = \frac{\|\mathbf{x}_{m-1}\|_2^2}{n_{\min}}.
\end{align}
Under zero-concentrated DP~\cite{bun2016concentrated}, adding
$\mathcal{N}(0,\sigma^2)$ noise to each cell mean at level~$m$
provides $\rho_m$-zCDP with
$\rho_m = \|\boldsymbol{\delta}_m\|_2^2 / (2\sigma^2)$.
Since orthogonal projections are contractive,
$\|\mathbf{x}_{m-1}\|_2^2 - \|\mathbf{x}_m\|_2^2 \geq 0$ and
\begin{align}
  \sum_{m=1}^{\ell} \rho_m
  \leq \frac{1}{2\,n_{\min}\,\sigma^2}
  \sum_{m=1}^{\ell} \|\mathbf{x}_{m-1}\|_2^2
  \leq \frac{\ell\,\|\Delta y\|^2}{2\,n_{\min}\,\sigma^2}.
\end{align}
The key observation is that the $L_2$ energies telescope:
$\sum_m \|\mathbf{x}_{m-1}\|_2^2 \leq \ell\,\|\mathbf{x}_0\|_2^2
= \ell\,\|\Delta y\|^2$.
In fact, since $\|\mathbf{x}_m\|_2 \leq \|\mathbf{x}_{m-1}\|_2$,
we can tighten this to
$\sum_m \|\mathbf{x}_{m-1}\|_2^2 \leq \ell\,\|\Delta y\|^2$,
giving $\rho_{\mathrm{total}} \leq
\ell\,\|\Delta y\|^2 / (2\,n_{\min}\,\sigma^2)$.
Converting to $(\varepsilon,\delta)$-DP via
$\varepsilon = \rho + \sqrt{2\rho\ln(1/\delta)}$
yields the stated bound.
\end{proof}

\begin{remark}
When the residual energy decays rapidly (empirically $\eta_m \geq 0.3$),
the sum $\sum_m \|\mathbf{x}_{m-1}\|_2^2$ is dominated by the first
few levels and the effective $\ell$ factor is much smaller than the
nominal depth.
\end{remark}

Adversarial simulations provide strong evidence that the true bound is
even tighter. We tested all code assignments exhaustively for
$N \leq 6$, and ran random ($50{,}000$ trials) and greedy adversarial
search ($200$ restarts) for $N \leq 500$, $b \leq 16$, $\ell \leq 12$
(millions of configurations total). We measured the per-level ratio
$|\delta_{m,j}| / (\|\Delta y\| / n_{\min})$ at each level
independently.

\emph{Depth-independent sensitivity.} Across all
configurations, the maximum per-level ratio remains between $0.95$ and
$1.25$ at every level from~$1$ to~$12$, with no upward trend.
Table~\ref{tab:dp_simulation} summarizes the per-level maxima for
representative configurations. This suggests a per-level bound of the
form $|\delta_{m,j}| \leq C\,\|\Delta y\| / n_{m,j}$ with $C \approx
1.25$, independent of depth. If this bound holds, the total privacy
cost under the Gaussian mechanism would be
$\rho_{\mathrm{total}} \leq C^2\,\ell\,\|\Delta y\|^2 /
(2\,n_{\min}^2\,\sigma^2)$,
which is linear in depth and inversely proportional to
$n_{\min}^2$.

\begin{table}[H]
\centering
\caption{Maximum per-level sensitivity ratio $|\delta_{m,j}| /
(\|\Delta y\| / n_{\min})$ across $30{,}000$ random code assignments
per configuration. The ratio stays flat across levels.}
\label{tab:dp_simulation}
\begin{small}
\begin{tabular}{rrrcccccc}
\toprule
$n$ & $b$ & $\ell$ & L1 & L2 & L4 & L6 & L8 & L10 \\
\midrule
50  & 4  & 6  & 1.00 & 0.95 & 1.12 & 1.20 & --- & --- \\
100 & 5  & 8  & 1.00 & 0.97 & 1.09 & 1.12 & 1.16 & --- \\
200 & 8  & 10 & 1.00 & 0.97 & 1.08 & 1.09 & 1.10 & 1.10 \\
500 & 16 & 12 & 1.00 & 0.98 & 1.01 & 1.00 & 1.02 & 0.96 \\
\bottomrule
\end{tabular}
\end{small}
\end{table}

Figure~\ref{fig:dp_sensitivity} visualizes the per-level ratio across
configurations: the empirical sensitivity stays flat near~$1$, far
below both the OGMD bound ($O(m)$, dotted) and the Cauchy--Schwarz
ceiling ($O(\sqrt{n_{\min}})$, dashed). Proving the $O(\ell/n_{\min})$
bound rigorously (or constructing a counterexample) remains open.

\begin{figure}[H]
\centering
\includegraphics[width=\linewidth]{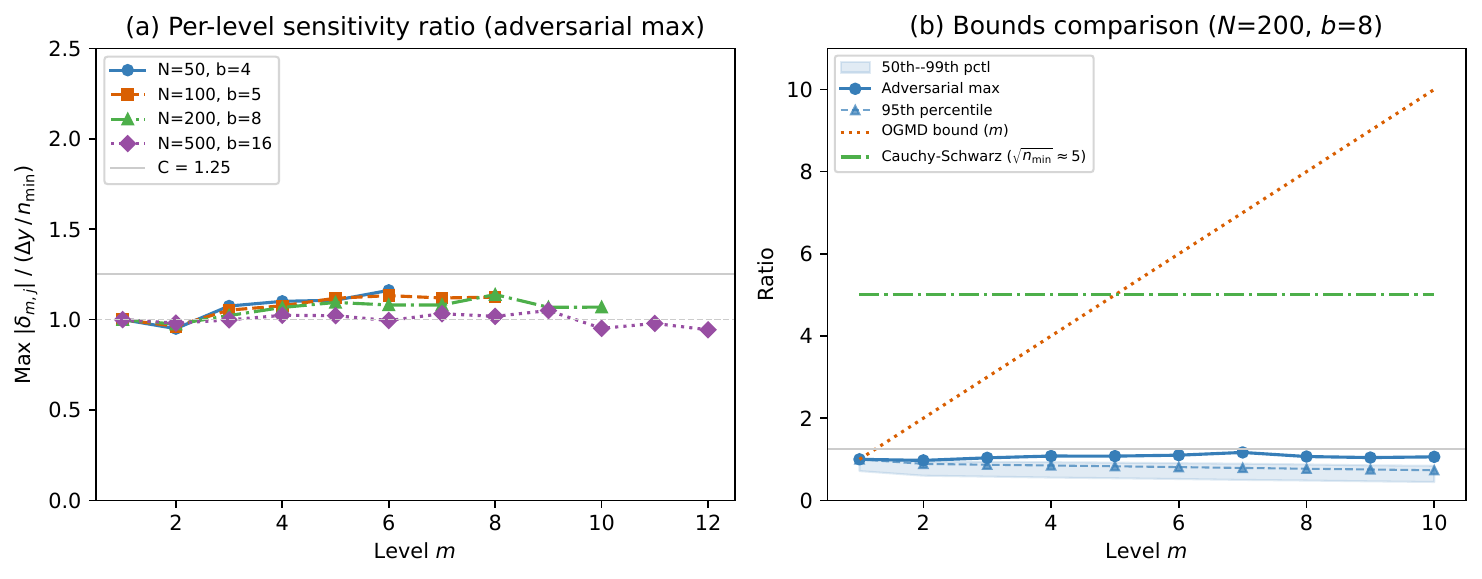}
\caption{Per-level sensitivity ratio
$|\delta_{m,j}|/(\|\Delta y\|/n_{\min})$. \textbf{(a)}~Adversarial
maximum across $20{,}000$ random code assignments for four
configurations. The ratio stays flat near~$1$ at all levels.
\textbf{(b)}~Comparison with theoretical bounds for $N{=}200$,
$b{=}8$: the empirical sensitivity (solid) is far below both the
na\"ive recurrence bound (dotted, grows linearly) and the Cauchy--Schwarz ceiling
(dashed, $\sqrt{n_{\min}} \approx 5$).}
\label{fig:dp_sensitivity}
\end{figure}

Empirically, membership inference (LiRA~\cite{carlini2022membership})
achieves AUC~$\approx 0.52$ across all noise levels, consistent with
the hypothesis that the cascade sensitivity is far below the
$2^{\ell-1}$ bound.

\subsection{Excess Risk Bound}
\label{app:excess_risk}

\begin{proposition}[MSE Decomposition]
\label{prop:excess_risk}
Let $\eta_m \geq \gamma > 0$ for all levels, and assume balanced clusters
($n_{m,j} \approx n/b$). The expected MSE of the finite-sample QSS estimator
$\hat{y}_\ell$ decomposes as:
\begin{align}
  \mathbb{E}[(\hat{y}_\ell - y)^2]
  \;\leq\; \underbrace{\mathrm{Var}(y)\prod_{m=1}^{\ell}(1-\eta_m)}_{\text{approximation}}
  \;+\; \underbrace{\frac{b\,\mathrm{Var}(y)(1-\gamma)}{n\gamma}}_{\text{estimation}}.
\end{align}
Under a Lipschitz condition on the target function, the
approximation term is $O(L^2 R^2 \rho^{2\ell})$.
\end{proposition}

\begin{proof}
Decompose the error through the oracle (infinite-data) predictor
$\bar{y}_\ell = \sum_m \mathbb{E}[r_{m-1} \mid c_m]$.
The approximation term is $\mathrm{Var}(r_\ell)$
(Proposition~\ref{prop:variance_decay}).

For the estimation term, each level's estimator $\hat{g}_m$ estimates the
cell mean with $\tilde{n} \approx n/b$ samples:
\begin{align}
  \mathrm{Var}(\hat{g}_m)
  &= \frac{1}{\tilde{n}}\,\mathbb{E}\bigl[\mathrm{Var}(r_{m-1} \mid c_m)\bigr]
  \nonumber \\
  &= \frac{b}{n}(1-\eta_m)\,\mathrm{Var}(r_{m-1}).
\end{align}
Summing across levels and expanding $\mathrm{Var}(r_{m-1})$:
\begin{align}
  \sum_{m=1}^{\ell} \mathrm{Var}(\hat{g}_m)
  &= \frac{b}{n} \sum_{m=1}^{\ell}
     (1{-}\eta_m)\,\mathrm{Var}(y)\prod_{j=1}^{m-1}(1{-}\eta_j)
     \nonumber \\
  &= \frac{b\,\mathrm{Var}(y)}{n}
     \sum_{m=1}^{\ell}\prod_{j=1}^{m}(1{-}\eta_j)
     \nonumber \\
  &\leq \frac{b\,\mathrm{Var}(y)}{n}
     \sum_{m=1}^{\ell}(1{-}\gamma)^m \nonumber \\
  &\leq \frac{b\,\mathrm{Var}(y)}{n}
     \cdot \frac{1{-}\gamma}{\gamma}.
\end{align}
The estimation term scales as $b/n$, \emph{not} $b^\ell/n$: the
label-residual structure uses each level independently ($b$ cells per level,
$n/b$ samples each), avoiding the curse of dimensionality that would arise
from joint $b^\ell$-cell estimation.
\end{proof}

\noindent\textbf{Interpretation and codebook collapse.}
The bound reveals two collapse modes.
\emph{Width collapse:} the optimal $b^*$ balances approximation (decreasing
in $b$) against estimation (increasing as $b/n$). When $b \to n$, cells
receive too few points and the per-cell means become noisy; at the extreme,
empty cells ($n_{m,j} = 0$) require smoothing but cannot recover signal.
\emph{Depth collapse:} if preconditioning concentrates between-class
variance into the first few dimensions, $\eta_m \to 0$ for later levels
and those levels are wasted, adding estimation error
without reducing approximation error. The practical rule is $b \ll n$ and
increase $\ell$ only while $\eta_m$ remains appreciable.


\subsection{Complexity of the QSS Method (Table~\ref{tab:methods})}
\label{app:qss_complexity}

\begin{proposition}[QSS Training Cost]
\label{prop:train_cost}
Training the full QSS method costs $t_g(n^{\circ}) + O(n\ell b d)$.
\end{proposition}

\begin{proof}
Training proceeds in two stages.
\emph{Stage~1:} train the base model~$g$ on $\mathcal{D}^{\circ}$
($|\mathcal{D}^{\circ}| = n^{\circ}$), costing $t_g(n^{\circ})$ (model-dependent).
\emph{Stage~2:} build the RQA data structure.
Encoding each of the $n$ points requires $\ell$ sequential levels;
at each level, the nearest of $b$ centroids is found by computing
$b$ inner products in $\mathbb{R}^d$, costing $O(bd)$ per level.
Total encoding: $O(n\ell b d)$.
The scatter-add pass that fills $\mathbf{T}_{m,j}$, $n_{m,j}$, and
$N_{k,m}$ costs $O(n\ell^2)$ (each point updates $\ell$ target sums
and $\binom{\ell}{2}$ co-occurrence counts).
Materialization costs $O(\ell^2 b^2 k)$
(Appendix~\ref{app:materialization}).
Since $bd \gg \ell$ in practice, the encoding term dominates and the
total is $t_g(n^{\circ}) + O(n\ell b d)$.
\end{proof}

\begin{proposition}[QSS Prediction Cost]
\label{prop:predict_cost}
Predicting with the full QSS method costs $t_g + O(\ell b d)$ per query.
\end{proposition}

\begin{proof}
Prediction (Equation~\ref{eq:predict}) evaluates
$\hat{y}(x) = g(x) + \alpha \cdot \textsc{RQA.Predict}(x)$.
The base-model inference $g(x)$ costs $t_g$.
\textsc{RQA.Predict} encodes the query ($O(\ell b d)$, same
argument as above) and sums $\ell$ cached cell means ($O(\ell k)$).
The total is $t_g + O(\ell b d)$.
\end{proof}

\begin{proposition}[Expected Unlearning Cost]
\label{prop:expected_unlearn}
For a uniformly random deletion, the expected cost is
$(1 - n^{\circ}/n)\,O(\ell^2 b^2 k) + (n^{\circ}/n)\,t_{\textup{rebuild}}$.
\end{proposition}

\begin{proof}
A deletion request targets a uniformly random point in
$\mathcal{D}$.
With probability $1 - n^{\circ}/n$, the point lies in
$\mathcal{D}^{\bullet}$: \textsc{RQA.Unlearn} decrements
$O(\ell^2)$ counters and re-materializes in $O(\ell^2 b^2 k)$
(Appendix~\ref{app:materialization}).
With probability $n^{\circ}/n$, the point lies in
$\mathcal{D}^{\circ}$: $g$ and the codebook must be retrained,
costing $t_{\textup{rebuild}}$.
The expectation follows by linearity.
\end{proof}

\subsection{Amortized Deletion and CTMC Analysis}
\label{app:amortized_proofs}

\begin{proof}[Proof of Proposition~\ref{prop:amortized}]
Each deletion hits instance~$i$'s $\mathcal{D}^{\circ}_i$
independently with probability~$\rho$.  A synchronous rebuild
occurs when all $p$~instances are invalidated.  The expected
number of deletions until all $p$ are hit follows the coupon
collector with non-uniform probabilities; for uniform~$\rho$,
it reduces to $H_p / \rho$ by a standard coupon-collector
argument.  Dividing the rebuild cost by this expected inter-rebuild
count gives the amortized term.
\end{proof}

\paragraph{CTMC steady state.}
Model the system as a continuous-time Markov chain on the number
of active instances $\{0, 1, \ldots, p\}$, with invalidation
rate~$\lambda\rho$ per active instance and repair rate~$\mu$ per
inactive instance (background rebuild).  The steady-state
probability of full outage (state~$0$) is
\begin{align}
  \pi_0 = \left(\sum_{j=0}^{p} \binom{p}{j}
    \frac{(\lambda\rho)^{p-j}\,\mu^{j}}
         {\prod_{i=1}^{p-j}(i\lambda\rho)
          \prod_{i=1}^{j}(i\mu)}\right)^{-1}
  \!\!= O\!\left(\left(\frac{\lambda\rho}{\mu}\right)^p\right).
\end{align}
For $\rho \ll 1$ and $\mu \gg \lambda\rho$ (rebuilds faster
than invalidations), full outages are exponentially rare in~$p$.
The optimal rebuild policy is analyzed in
Appendix~\ref{app:rebuild_proofs}.

\subsection{Rebuild Scheduling (Proposition~\ref{prop:rebuild})}
\label{app:rebuild_proofs}

We analyze the non-blocking rebuild model of
Proposition~\ref{prop:rebuild}.  The system maintains $p$~independent
schemas.  A deletion request independently invalidates each active
instance with probability~$\rho$ (since $\mathcal{D}^{\circ}$
membership is secret and uniformly random).  A rebuild costs
$t_{\textup{rebuild}}$ in CPU, takes wall-clock time~$T_r$, restores
all $p$~instances, and does not interrupt serving by surviving
instances.  At most one rebuild is in flight at a time.
The \emph{eager} policy starts a background rebuild whenever any
instance is invalidated, if none is already in progress.

\begin{proof}[Proof of Proposition~\ref{prop:rebuild}(i): Outage
  Bound]

When eager triggers a rebuild, exactly $p{-}1$ instances remain active
(the one just invalidated plus the in-progress rebuild).
During the rebuild window~$T_r$, the number of arriving deletions is
$D \sim \mathrm{Poisson}(\lambda_d T_r)$.  For each deletion,
each surviving instance is independently invalidated with
probability~$\rho$.  Conditional on~$D$, instance~$i$ survives all
$D$~deletions with probability $(1{-}\rho)^D$, and these events are
independent across instances.

\paragraph{Marginal survival probability.}
The unconditional survival probability of a single instance is:
\begin{align}
  \alpha
  \;=\; \mathbb{E}\bigl[(1{-}\rho)^D\bigr]
  \;=\; e^{-\lambda_d T_r\,(1-(1-\rho))}
  \;=\; e^{-\rho\lambda_d T_r}.
  \label{eq:alpha}
\end{align}

\paragraph{Joint outage probability.}
A full outage requires all $p{-}1$ surviving instances to die.
By inclusion-exclusion:
\begin{align}
  P_{\textup{out}}
  &= \sum_{k=0}^{p-1} (-1)^k \binom{p{-}1}{k}\,
     \mathbb{E}\bigl[(1{-}\rho)^{kD}\bigr] \nonumber \\
  &= \sum_{k=0}^{p-1} (-1)^k \binom{p{-}1}{k}\,
     e^{-\lambda_d T_r(1-(1-\rho)^k)}.
  \label{eq:pout_exact}
\end{align}

\paragraph{Lower bound.}
Since $f(x) = x^{p-1}$ is convex for $p \geq 2$ and
$q(D) = 1{-}(1{-}\rho)^D \in [0,1]$:
\begin{align}
  P_{\textup{out}}
  = \mathbb{E}\bigl[q(D)^{p-1}\bigr]
  \geq \bigl(\mathbb{E}[q(D)]\bigr)^{p-1}
  = (1{-}\alpha)^{p-1}.
  \label{eq:pout_lower}
\end{align}

\paragraph{Upper bound and tightness.}
The gap between~\eqref{eq:pout_exact} and~\eqref{eq:pout_lower}
arises from the positive correlation: a large~$D$ realization
simultaneously increases all instances' death probabilities.
To quantify this, note that the joint survival probability
of $k$~instances is:
\begin{align}
  \mathbb{E}[(1{-}\rho)^{kD}]
  = e^{-\lambda_d T_r(1-(1-\rho)^k)}
  = \alpha^k \cdot e^{\lambda_d T_r\,(k\rho - (1-(1-\rho)^k))}.
\end{align}
The correction exponent satisfies
$k\rho - (1{-}(1{-}\rho)^k) = \binom{k}{2}\rho^2 + O(k^3\rho^3)$.
Substituting into~\eqref{eq:pout_exact} and comparing
with the binomial expansion of $(1{-}\alpha)^{p-1}$, the
multiplicative gap is:
\begin{align}
  \frac{P_{\textup{out}}}{(1{-}\alpha)^{p-1}}
  = 1 + O\!\left(\tbinom{p{-}1}{2}\,\rho^2 \lambda_d T_r\right).
\end{align}
For $\rho = 0.5\%$, $p = 3$, $\lambda_d T_r = 10$:
the correction is $1 + O(2.5 \times 10^{-4}) \approx 1.0003$.
In the practical regime ($\rho \ll 1$), the lower
bound~\eqref{eq:pout_lower} is therefore tight, and
\begin{align}
  P_{\textup{out}} \approx (1 - e^{-\rho\lambda_d T_r})^{p-1}.
\end{align}
For $\rho{=}0.5\%$ and $\lambda_d T_r{=}10$:
$P_{\textup{out}} \approx 0.049^{p-1}$.  At $p{=}3$:
$P_{\textup{out}} \approx 0.24\%$.
\end{proof}

\begin{proof}[Proof of Proposition~\ref{prop:rebuild}(ii): Rebuild
  Overhead]

We compare the number of rebuilds performed by the eager policy
against the offline optimal (which knows the full deletion
sequence in advance).

\paragraph{Eager rebuild rate.}
Starting from $k = p$ (all instances active), the next invalidation
event has probability $1{-}(1{-}\rho)^p$ per deletion.  For
small~$\rho$, this is approximately~$p\rho$.  Under eager, every
invalidation triggers a rebuild, so the expected inter-rebuild gap
is:
\begin{align}
  N_{\textup{eager}}
  = \frac{1}{1{-}(1{-}\rho)^p}
  \approx \frac{1}{p\rho}
  \quad\text{deletions.}
  \label{eq:eager_gap}
\end{align}

\paragraph{Lower bound on any policy's rebuild rate.}
Without a rebuild, the instance count~$k$ can only decrease.
Starting from $k = p$, the expected number of deletions until all
$p$~instances are invalidated is a coupon-collector problem: at
state~$k$, the next death has probability
$1{-}(1{-}\rho)^k \approx k\rho$, giving expected waiting time
$\approx 1/(k\rho)$.  Summing:
\begin{align}
  N_{p \to 0}
  = \sum_{k=1}^{p} \frac{1}{1{-}(1{-}\rho)^k}
  \approx \frac{1}{\rho}\sum_{k=1}^{p}\frac{1}{k}
  = \frac{H_p}{\rho}.
  \label{eq:coupon}
\end{align}
Any policy, including the offline optimal, must rebuild at
least once every $H_p / \rho$ deletions on average, since
without a rebuild the system inevitably reaches $k = 0$.

\paragraph{Competitive ratio.}
The ratio of rebuild rates is:
\begin{align}
  \frac{N_{p \to 0}}{N_{\textup{eager}}}
  = \frac{H_p / \rho}{1/(p\rho)}
  = p\,H_p.
\end{align}
This bound is tight: an offline-optimal policy with instantaneous
rebuilds ($T_r = 0$) can wait until $k = 0$ and rebuild without
any outage, achieving inter-rebuild gap $H_p / \rho$.  For
$p = 3$: $p\,H_3 = 3 \cdot (1 + \tfrac{1}{2} + \tfrac{1}{3})
= 5.5$; for $p = 2$: $p\,H_2 = 3$.
\end{proof}

\paragraph{Steady-state cost rate.}
To choose among rebuild thresholds, we
generalize to the threshold-$k^*$ policy: trigger a rebuild when the
active instance count drops to~$k^*$.  The long-run average cost per
deletion follows from renewal theory (ratio of expected cost per
cycle to expected cycle length):
\begin{align}
  \mathcal{C}(k^*)
  = \frac{t_{\textup{rebuild}}
    + P_{\textup{out}}(k^*) \cdot C \cdot \lambda_p \cdot T_r}
    {N_{p \to k^*} + \lambda_d \cdot T_r},
  \label{eq:cost_rate}
\end{align}
where
$N_{p \to k^*} = \sum_{k=k^*+1}^{p} [1{-}(1{-}\rho)^k]^{-1}$
is the expected number of deletions from $k{=}p$ to $k{=}k^*{+}1$,
and $P_{\textup{out}}(k^*) \approx (1{-}e^{-\rho\lambda_d T_r})^{k^*}$
is the outage probability with $k^*$~instances surviving during the
rebuild window.

In the regime $C \gg t_{\textup{rebuild}}$ (outage is
catastrophic), the second term in the numerator dominates and
$k^* = p{-}1$ (eager) minimizes~\eqref{eq:cost_rate}.
When $C$ is moderate, the optimal threshold balances rebuild
frequency against outage risk; for small~$p$, this can be evaluated
numerically.

\subsection{RQA Pseudocode}
\label{app:rqa_pseudocode}

\begin{algorithm}[t]
\caption{The RQA data structure. RQ.encode is the frozen
Faiss encoder; the four RQA procedures operate on
target values $\mathbf{Y}^r$; in QSS, $y^r_i = y_i - g(x_i)$
(\S\ref{sec:model-pair}).}
\label{alg:qss}
\small
\textbf{Input:} Feature map $\varphi$, RQ with $b$
centroids/level, $\ell$ levels\\
\textbf{State:}\;
$\mathbf{T}_{m,j} {\in} \mathbb{R}^k,\;
n_{m,j} {\in} \mathbb{N},\;
N_{k,m} {\in} \mathbb{N}^{b \times b}$.
\;Cache\; $\mathbf{S}_{m,j} {\in} \mathbb{R}^k$.
\vspace{1mm}
\hrule
\vspace{1mm}
\begin{minipage}[t]{0.48\textwidth}
\begin{algorithmic}[1]
\Procedure{RQA.Build}{$\mathbf{X}, \mathbf{Y}^r$}
  \State $\mathbf{C} \leftarrow \text{RQ.encode}(\varphi(\mathbf{X}))$
    \Comment{$n \times \ell$}
  \For{$m = 1$ \textbf{to} $\ell$}
    \State $\mathbf{T}_{m,j} \leftarrow
      \sum_{i:\,C_{i,m}=j} y^r_i$
      \Comment{scatter-add}
    \State $n_{m,j} \leftarrow |\{i : C_{i,m} = j\}|$
  \EndFor
  \For{$k = 1$ \textbf{to} $\ell$;\; $m = k{+}1$ \textbf{to} $\ell$}
    \State $N_{k,m}[j',j] \leftarrow
      |\{i : C_{i,k}{=}j',\, C_{i,m}{=}j\}|$
  \EndFor
  \State \Call{RQA.Materialize}{}
\EndProcedure
\Statex
\Procedure{RQA.Predict}{$x$}
  \State $\mathbf{c} \leftarrow \text{RQ.encode}(\varphi(x))$
  \State \Return $\sum_{m=1}^{\ell}
    \mathbf{S}_{m, c_m} / n_{m, c_m}$
\EndProcedure
\end{algorithmic}
\end{minipage}%
\hfill\vrule\hfill%
\begin{minipage}[t]{0.48\textwidth}
\begin{algorithmic}[1]
\setcounter{ALG@line}{13}
\Procedure{RQA.Unlearn}{$x, y^r$}
  \State $\mathbf{c} \leftarrow \text{RQ.encode}(\varphi(x))$
  \For{$m = 1$ \textbf{to} $\ell$}
    \State $\mathbf{T}_{m, c_m} \mathrel{-}= y^r$;\;
      $n_{m, c_m} \mathrel{-}= 1$
    \For{$k = 1$ \textbf{to} $m{-}1$}
      \State $N_{k,m}[c_k, c_m] \mathrel{-}= 1$
    \EndFor
  \EndFor
  \State \Call{RQA.Materialize}{}
\EndProcedure
\Statex
\Procedure{RQA.Materialize}{}
  \Comment{$O(\ell^2 b^2 k)$}
  \For{$m = 1$ \textbf{to} $\ell$}
    \State $\mathbf{S}_{m,j} \leftarrow \mathbf{T}_{m,j}
      {-} \sum_{k<m} \sum_{j'}
        N_{k,m}[j'\!,j] \cdot \boldsymbol{\mu}_{k,j'}$
    \State $\boldsymbol{\mu}_{m,j} \leftarrow
      \mathbf{S}_{m,j} / n_{m,j}$
      \;($\mathbf{0}$ if $n_{m,j} {=} 0$)
  \EndFor
\EndProcedure
\end{algorithmic}
\end{minipage}
\end{algorithm}

\section{Codebook Fragility: Extended Analysis}
\label{app:fragility}

We train RQ ($b{=}1024$, $\ell{=}8$) on 162,770 CelebA CLIP embeddings, remove
100 random points (0.06\%), retrain, and re-encode. Result: 100\% of remaining
codes change (per-level flip rate 99.8--99.9\%, 3 repeats). Minimum Voronoi
margin is $\approx 0$ at every level; cross-level dilution is 1.0, confirming
cascades propagate at full strength.

\begin{figure}[H]
\centering
\includegraphics[width=\linewidth]{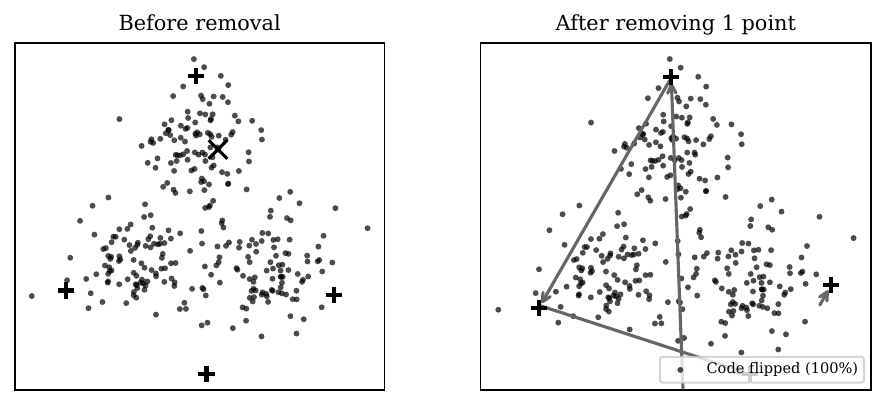}
\caption{Codebook instability (\S\ref{app:fragility})
on 2D synthetic data ($b{=}20$, 300 points). \textbf{Left:}
original $k$-means centroids ($+$) and Voronoi boundaries.
\textbf{Right:} removing one point ($\times$) and retraining
$k$-means shifts centroids (arrows), flipping 30\% of codes
(dark points). Flips occur far from the removed point,
demonstrating that the perturbation is global.
In 768-d CLIP embeddings, the flip rate is 100\%.}
\label{fig:fragility}
\end{figure}


\section{Additional Experiments}
\label{app:additional_experiments}

\subsection{Aggregate Accuracy Patterns}

\begin{figure}[H]
\centering
\includegraphics[width=\linewidth]{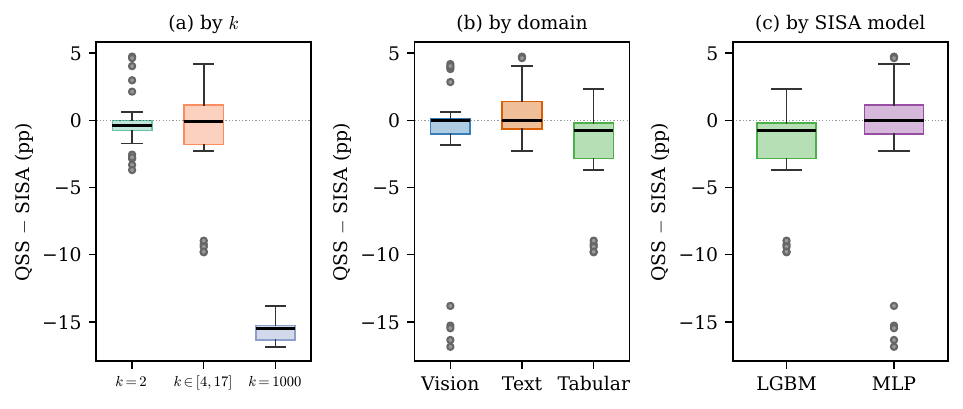}
\caption{Aggregate view across all 15 datasets ($\rho{=}0.5\%$,
per-seed).
\textbf{(a)}~QSS-L--SISA accuracy gap by $k$: binary and
low-$k$ tasks cluster near zero; high-$k$ (ImageNet) is a
clear outlier (\S\ref{sec:limitations}).
\textbf{(b)}~Gap by domain: all three modalities behave
similarly.
\textbf{(c)}~Gap by SISA model type (LGBM,
MLP): QSS is competitive regardless of baseline
architecture.}
\label{fig:aggregate}
\end{figure}

\subsection{Full Accuracy and Speedup}
\label{app:full_summary}

Fixed QSS-L and QSS-E accuracy for all 15 datasets appears in
Table~\ref{tab:fixed_modes}; the low-$k$ QSS-L/SISA comparison appears in
Table~\ref{tab:summary}. We do not report a best-over-guarantees aggregate.

\begin{figure}[H]
\centering
\includegraphics[width=\linewidth]{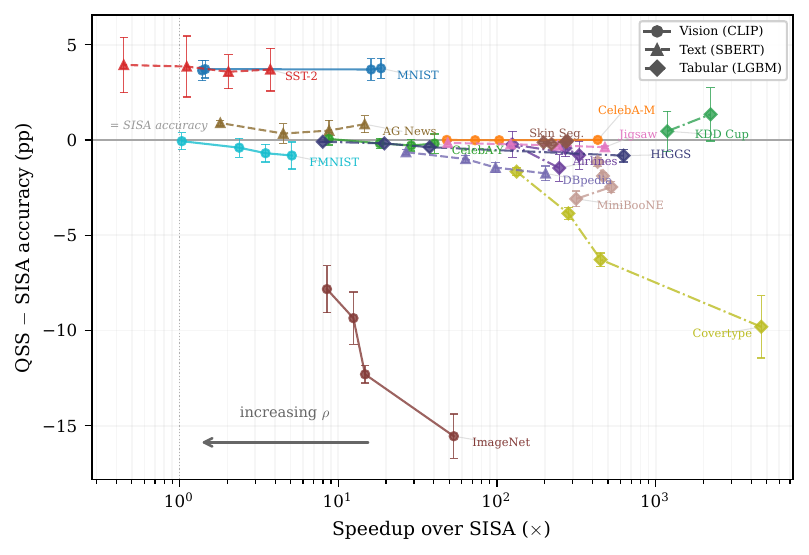}
\caption{QSS-L accuracy--speedup Pareto as $\rho$ varies from
$0.5\%$ to $5\%$, all 15 datasets. Both axes are relative to SISA
($x{=}1$, $y{=}0$). Each trajectory traces one dataset;
moving \emph{left} corresponds to increasing $\rho$
(lower speedup, higher accuracy).
Gaps on ImageNet and Covertype (high-$k$ tasks) narrow
with larger~$\rho$.
Full timing breakdown in Table~\ref{tab:timing}.}
\label{fig:pareto_d0_full}
\end{figure}

\subsection{Expected Unlearning Latency}
\label{app:timing}

\begin{table}[H]
\centering
\caption{Expected unlearning latency at $\rho{=}0.5\%$.
$\mathbb{E}[\text{QSS-L}] =
(1{-}\rho)\,t_{\text{RQ}} + \rho\,t_{\text{rebuild}}$
(Equation~\ref{eq:pareto});
SISA ($S{=}5$) retrains one shard.
$t_{\text{rebuild}}$ includes retraining $g$, re-extracting
embeddings (for LGBM: leaf contributions), re-fitting LDA and
codebook, re-encoding all $n$ points, and re-accumulating.
Per-component RQ timings appear in
Table~\ref{tab:rq_timing}.
$^\dagger$High-$k$ limitation (\S\ref{sec:limitations}).}
\label{tab:timing}
\smallskip
\begin{small}
\begin{tabular}{lrrrrrr}
\toprule
Dataset & $n$ & $\rho$ & $t_{\text{reb}}$ (s)
  & $\mathbb{E}[\text{QSS-L}]$ & SISA & Speedup \\
\midrule
Jigsaw             & 1.4M  & 0.5\% &     44 &   0.2\,s & 109\,s   & 483$\times$ \\
CelebA Male        &  137K & 0.5\% &     13 &   0.1\,s & 9.5\,s   & 132$\times$ \\
MiniBooNE          &  130K & 0.5\% &     30 &   0.2\,s & 17.5\,s  & 114$\times$ \\
DBpedia            &  473K & 0.5\% &     29 &   0.1\,s & 15.8\,s  & 107$\times$ \\
Covertype$^\dagger$&  581K & 0.5\% &    850 &   4.3\,s & 98.7\,s  &  23$\times$ \\
Airlines           &  500K & 0.5\% &     50 &   0.3\,s & 11.0\,s  &  54$\times$ \\
AG~News            &   96K & 0.5\% &     14 &  75\,ms  & 3.8\,s   &  51$\times$ \\
ImageNet$^\dagger$ &  998K & 0.5\% &    227 &   1.1\,s & 52.9\,s  &  46$\times$ \\
MNIST              &   53K & 0.5\% &     13 &  68\,ms  & 3.1\,s   &  46$\times$ \\
CelebA Young       &  137K & 0.5\% &     17 &  88\,ms  & 3.6\,s   &  41$\times$ \\
Skin Seg.          &  245K & 0.5\% &    126 &   0.6\,s & 12.0\,s  &  19$\times$ \\
FMNIST             &   53K & 0.5\% &     64 &   0.3\,s & 2.5\,s   &   8$\times$ \\
HIGGS              &   11M & 0.5\% &  3{,}522 &  17.6\,s & 114\,s   &   6$\times$ \\
KDD~Cup            &  494K & 0.5\% &  2{,}381 &  11.9\,s & 154\,s   &   6$\times$ \\[1pt]
SST-2              &   51K & 0.5\% &     63 &   0.3\,s & 1.2\,s   &   4$\times$ \\
\bottomrule
\end{tabular}
\end{small}
\end{table}

\subsection{Codebook Collapse}
\label{app:codebook_collapse}

\begin{figure}[H]
\centering
\includegraphics[width=0.7\linewidth]{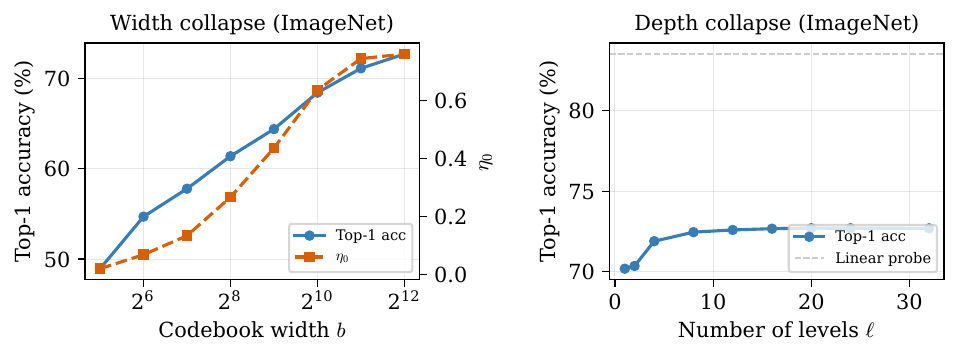}
\caption{Codebook collapse on ImageNet (unsupervised RQ).
Left: width collapse (accuracy and $\eta_0$ grow with $b$).
Right: depth collapse (accuracy plateaus as $\eta_\ell \to 0$).
See Appendix~\ref{app:proof_variance} for the variance decay
analysis that explains this convergence behavior.}
\label{fig:collapse}
\end{figure}

\subsection{Accuracy Retention Under Deletion}

\begin{table}[H]
\centering
\caption{Accuracy retention under random deletion. Unlearn and
retrain-from-scratch produce bitwise identical predictions at every
deletion fraction (Proposition~\ref{prop:exact_unlearning}).}
\label{tab:retention}
\smallskip
\begin{small}
\begin{tabular}{lccccccccc}
\toprule
 & 0\% & 1\% & 5\% & 10\% & 20\% & 30\% & 40\% & 50\% \\
\midrule
CelebA     & 99.35 & 99.35 & 99.34 & 99.36 & 99.33 & 99.33 & 99.32 & 99.33 \\
ImageNet   & 72.67 & 72.67 & 72.66 & 72.65 & 72.66 & 72.65 & 72.48 & 72.44 \\
Covertype  & 75.45 & 75.39 & 75.40 & 75.38 & 75.36 & 75.33 & 75.21 & 75.14 \\
AG News    & 90.08 & 90.11 & 90.05 & 90.14 & 89.95 & 89.89 & 89.89 & 89.86 \\
\bottomrule
\end{tabular}
\end{small}
\end{table}

\noindent CelebA, ImageNet, Covertype, and AG~News are stable through
50\% deletion ($<$0.3pp loss). All four datasets are \emph{bitwise
identical} between unlearn and retrain-from-scratch at every fraction.

\subsection{Membership Inference Attack}
\label{app:mia}

We evaluate membership privacy via LiRA~\cite{carlini2022membership}:
64 shadow QSS models (each trained on a random 50\% subset of CelebA),
auditing 1000 target points.

\begin{table}[H]
\centering
\caption{LiRA membership inference on QSS (CelebA Male).
The attack barely outperforms random guessing at all $\varepsilon$,
including $\varepsilon{=}\infty$ (no DP noise).}
\label{tab:lira}
\smallskip
\begin{small}
\begin{tabular}{lccc}
\toprule
$\varepsilon$ & ROC AUC & TPR@FPR=0.1\% & TPR@FPR=1\% \\
\midrule
1.0            & 0.515 & 0.2\% & 2.2\% \\
2.0            & 0.518 & 0.8\% & 1.2\% \\
4.0            & 0.531 & 0.2\% & 1.4\% \\
$\infty$ (no DP) & 0.525 & 0.0\% & 0.4\% \\
\bottomrule
\end{tabular}
\end{small}
\end{table}

\noindent Even without DP, QSS leaks no membership information:
AUC~$\approx 0.52$. Predictions depend on cell-level aggregates
($n_{m,j} \gg 1$), not individual points.

\begin{figure}[H]
\centering
\includegraphics[width=0.6\linewidth]{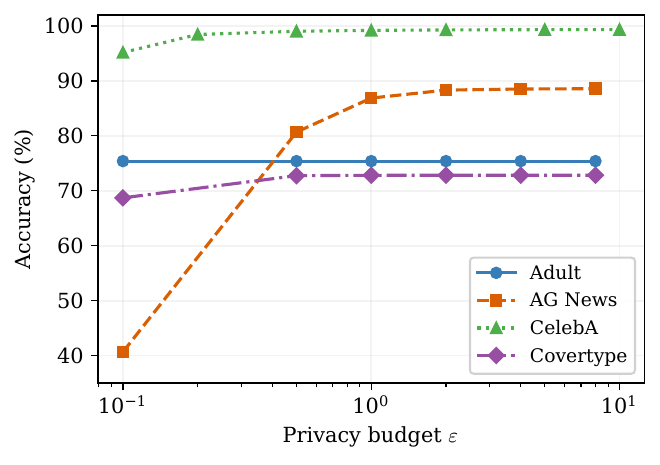}
\caption{$\varepsilon$-label-DP accuracy. Adult (binary) is
insensitive; Covertype (7-class) and AG~News (4-class) degrade
gracefully, retaining most accuracy at $\varepsilon \geq 1$.}
\label{fig:epsilon_accuracy_full}
\end{figure}

\subsection{Variance Decay and Theory Validation}
\label{app:variance_decay_empirical}

Proposition~\ref{prop:variance_decay} (Variance Decay) states that
$\mathrm{Var}(r_m) = \mathrm{Var}(y^r)\,\prod_{j=1}^{m}(1 - \eta_j)$.
If $\eta_m \geq \gamma > 0$, MSE decays as $(1{-}\gamma)^\ell$.
Proof in Appendix~\ref{app:proof_variance}.
Excess risk scales as $b/n$, not $b^\ell/n$
(Appendix~\ref{app:excess_risk}).

\begin{figure}[H]
\centering
\begin{minipage}[t]{0.48\linewidth}
  \centering
  \includegraphics[width=\linewidth]{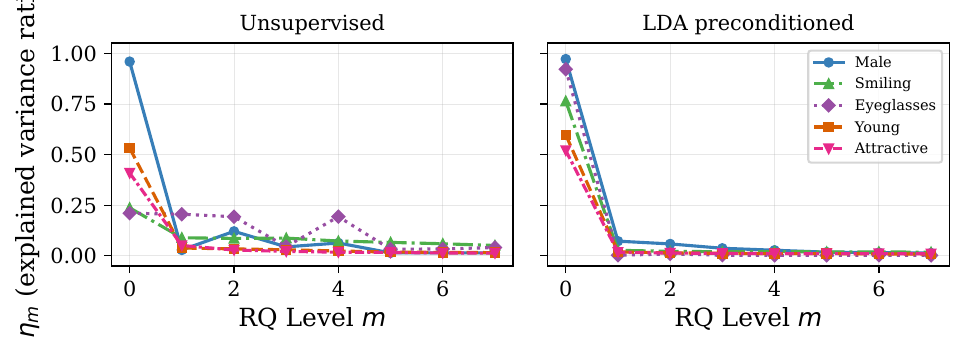}
  \caption{Per-level $\eta_m$ across datasets.}
  \label{fig:eta_per_level}
\end{minipage}\hfill
\begin{minipage}[t]{0.48\linewidth}
  \centering
  \includegraphics[width=\linewidth]{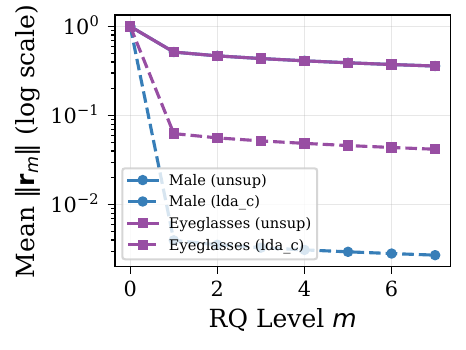}
  \caption{Residual norm decay across levels.}
  \label{fig:residual_decay}
\end{minipage}
\end{figure}

\subsection{Ablations}
\label{app:ablations}

We isolate the contribution of codebook width~$b$, depth~$\ell$,
and LDA preconditioning. All ablations use 5 seeds (0--4) on the
DNN pipeline (MNIST, CIFAR-100, AG~News, Covertype) at
$|\mathcal{D}^{\circ}| \in \{500, 5{,}000\}$.

\begin{figure}[H]
\centering
\includegraphics[width=\linewidth]{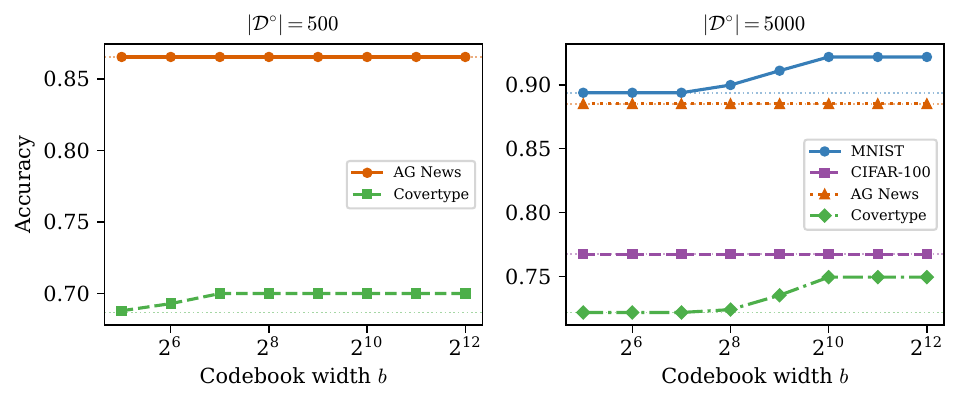}
\caption{Accuracy vs.\ codebook width $b$ ($\ell{=}4$ fixed).
Dotted lines show $g$-only baselines.
Width matters most on MNIST and Covertype at
$|\mathcal{D}^{\circ}|{=}5{,}000$, where larger $b$ provides
finer Voronoi cells that better separate classes.
AG~News and CIFAR-100 are flat ($\alpha \approx 0$).}
\label{fig:ablation_b}
\end{figure}

\begin{figure}[H]
\centering
\includegraphics[width=\linewidth]{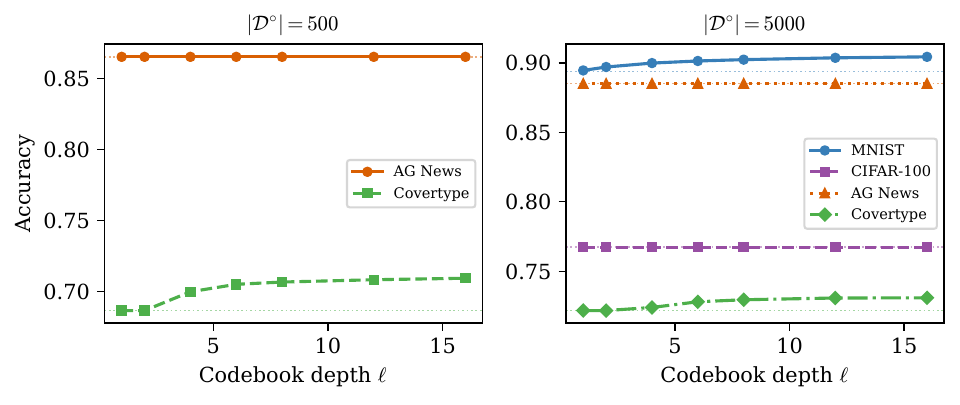}
\caption{Accuracy vs.\ codebook depth $\ell$ ($b{=}256$ fixed).
Covertype gains $\sim$2\,pp from $\ell{=}1$ to $\ell{=}8$;
other datasets saturate by $\ell{=}2$--$4$, consistent with the
depth collapse observed in Figure~\ref{fig:collapse}.}
\label{fig:ablation_ell}
\end{figure}

\begin{figure}[H]
\centering
\includegraphics[width=0.85\linewidth]{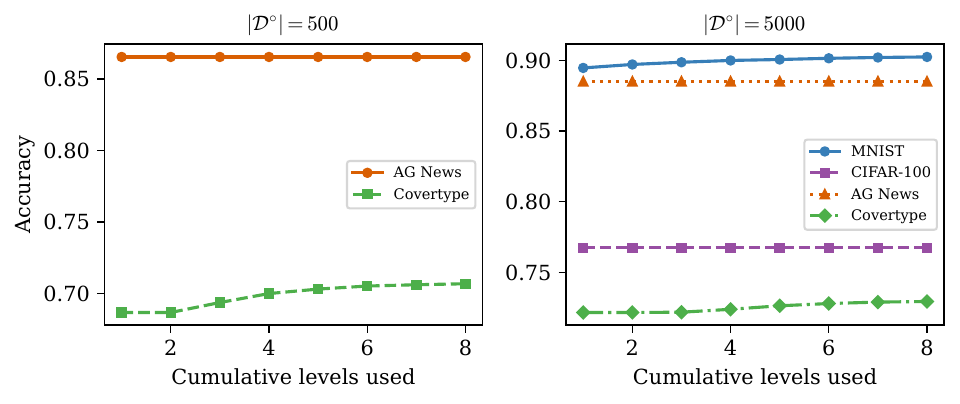}
\caption{Cumulative accuracy as RQ levels are added.
Level~1 captures most of the QSS lift; deeper levels provide
diminishing returns, confirming the variance decay prediction
(Proposition~\ref{prop:variance_decay}).}
\label{fig:ablation_per_level}
\end{figure}

\begin{figure}[H]
\centering
\includegraphics[width=0.85\linewidth]{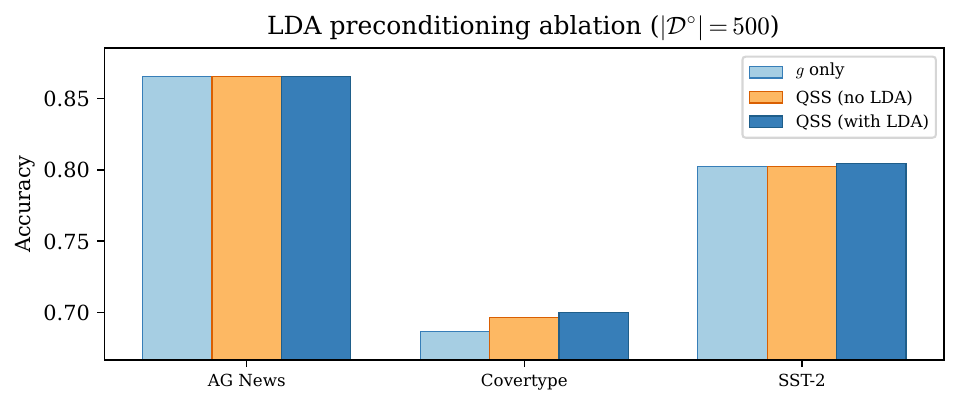}
\caption{LDA preconditioning ablation
($|\mathcal{D}^{\circ}|{=}500$). LDA provides a modest lift on
Covertype ($+1.5$\,pp) where $k{-}1{=}6$ discriminant directions
help the unsupervised codebook align with class boundaries.
Effect is negligible on AG~News and SST-2 where the pretrained
SBERT embeddings already separate classes well.}
\label{fig:ablation_lda}
\end{figure}

\subsection{Hyperparameter Sensitivity and Rebuild Efficiency}
\label{app:sensitivity}

The codebook parameters $(b, \ell)$ jointly control accuracy and
unlearning speed: $t_{\text{RQ}} = O(\ell^2 b^2 k)$
(Equation~\ref{eq:pareto}).
Cross-validation selects $(b, \ell, \alpha)$ to maximize held-out
accuracy on $\mathcal{D}^{\circ}$; here we examine the full
parameter surface to characterize robustness.

\paragraph{Accuracy is robust to $(b, \ell)$ when $g$ is strong.}
Figure~\ref{fig:pareto_bl} sweeps $b$ and $\ell$ independently on
two representative datasets.
On MiniBooNE (tabular, $g$-only $= .912$), QSS accuracy is flat
across the entire $(b, \ell)$ range: $(b{=}32, \ell{=}1)$ and
$(b{=}512, \ell{=}8)$ yield the same $.913$ accuracy,
because the blending weight $\alpha \approx 0$ lets the strong
base model dominate.
On FMNIST (vision, $g$-only $= .749$), QSS accuracy increases
with both $b$ and $\ell$: moving from $\ell{=}1$ to $\ell{=}8$
gains $1.5$\,pp, because the cell-mean correction contributes
meaningfully when $g$ is far from the ceiling.
The practical consequence is that datasets where $g$ already
performs well (typically those with high-quality embeddings) can
use minimal codebook parameters without sacrificing accuracy,
reducing $t_{\text{RQ}}$ by orders of magnitude
(e.g., $0.02$\,ms at $(32,1)$ vs.\ $9$\,ms at $(512,8)$).

\begin{figure}[H]
\centering
\includegraphics[width=\linewidth]{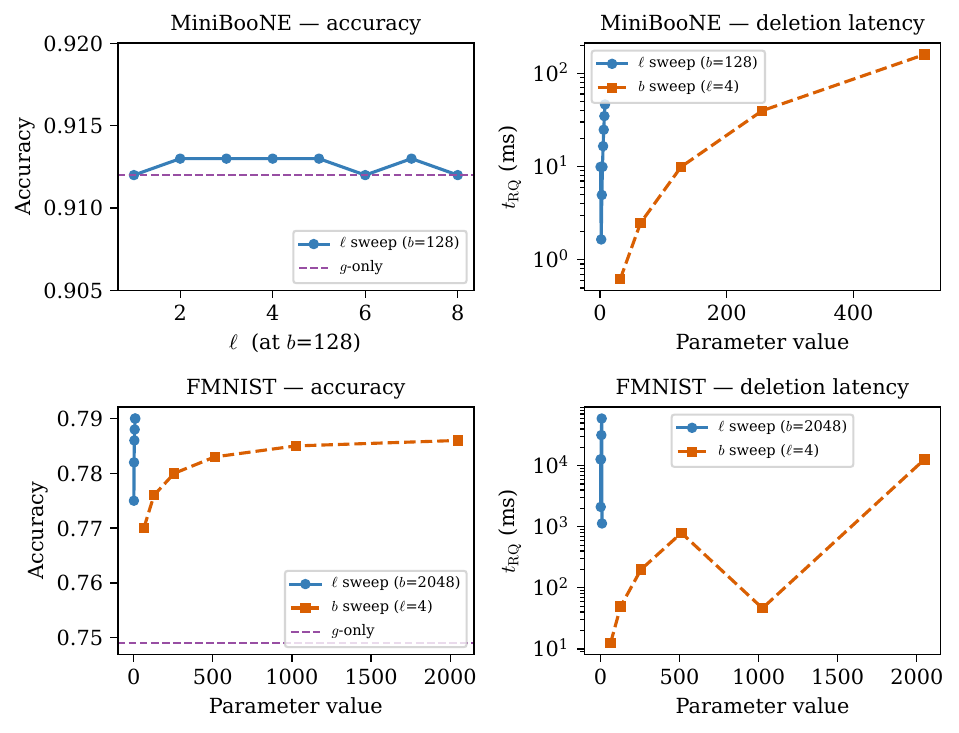}
\caption{QSS accuracy (left) and $t_{\text{RQ}}$ (right) as
codebook parameters vary.
\textbf{Top:} MiniBooNE ($\ell$ sweep at $b{=}128$; $b$ sweep
at $\ell{=}4$).  Accuracy is flat; $t_{\text{RQ}}$ spans
three orders of magnitude.
\textbf{Bottom:} FMNIST (same sweeps). Accuracy improves with
larger codebooks; we trade deletion speed for
accuracy via CV.}
\label{fig:pareto_bl}
\end{figure}

\paragraph{Accumulator convergence.}
A full rebuild re-encodes all $n$ points and re-accumulates their
labels into the joint-count tables.
Since these tables are sufficient statistics (cell-level sums and
counts), they converge as the per-cell sample sizes grow.
We measure accuracy as a function of the fraction of
$\mathcal{D}^{\bullet}$ encoded into accumulators
(Figure~\ref{fig:acc_convergence}).

On tabular datasets with well-clustered LGBM embeddings,
accumulators converge rapidly: Skin~Segmentation reaches full
accuracy at 1\% of $n$ (2.5K points), KDD~Cup at 5\%,
HIGGS at 10\%, and Airlines/MiniBooNE at 25\%
(Table~\ref{tab:convergence}).
This reflects high cell purity: LGBM leaf embeddings place
same-class points in tight clusters, so even a fraction of the
data yields stable per-cell label distributions.

On DNN datasets (CLIP, MiniLM), embeddings are more diffuse and
accumulators require the full dataset to converge.
FMNIST loses $1$\,pp at 50\% fill and $9.5$\,pp at 1\%;
SST-2 shows a similar pattern.
QSS genuinely leverages all available data in this regime.

\begin{table}[H]
\centering
\caption{Accumulator convergence on tabular datasets.
\emph{Saturation fraction}: smallest fraction of
$\mathcal{D}^{\bullet}$ at which accuracy is within 0.15\,pp
of full.  Subsampled rebuilds reduce the encode$+$accumulate
cost proportionally.}
\label{tab:convergence}
\smallskip
\begin{small}
\begin{tabular}{lrrrl}
\toprule
Dataset & $n$ & Saturation & Gap & Rebuild savings \\
\midrule
Skin Seg.    & 245K &  1\% (2.5K)  & $-$0.02\,pp & 100$\times$ \\
KDD Cup      & 494K &  5\% (25K)   & $-$0.15\,pp &  20$\times$ \\
HIGGS        &  11M & 10\% (1.1M)  & $-$0.02\,pp &  10$\times$ \\
Covertype    & 581K & 10\% (58K)   & $-$0.40\,pp &  10$\times$ \\
Airlines     & 500K & 25\% (125K)  & $-$0.05\,pp &   4$\times$ \\
MiniBooNE    & 130K & 25\% (33K)   & $-$0.15\,pp &   4$\times$ \\
\bottomrule
\end{tabular}
\end{small}
\end{table}

\begin{figure}[H]
\centering
\includegraphics[width=0.85\linewidth]{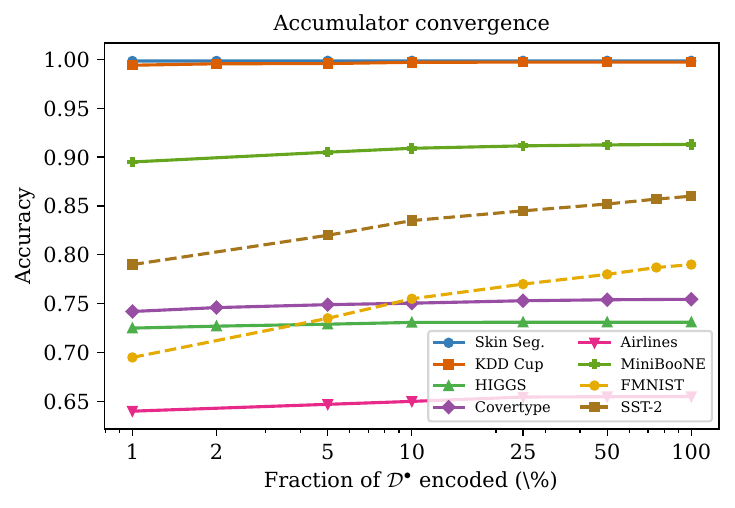}
\caption{Accuracy vs.\ fraction of $\mathcal{D}^{\bullet}$
encoded into accumulators.
Tabular datasets (solid) converge rapidly; DNN datasets (dashed)
require most or all of the data.
The convergence rate reflects embedding geometry:
well-clustered embeddings yield high cell purity and fast
convergence.}
\label{fig:acc_convergence}
\end{figure}

The convergence behavior has a direct practical implication for
rebuild cost.
The full rebuild pipeline has four stages: retrain~$g$, retrain
codebook (both on $\mathcal{D}^{\circ}$ only, independent of
$|\mathcal{D}^{\bullet}|$), encode $\mathcal{D}^{\bullet}$, and
accumulate.
Only the last two stages scale with the number of points encoded.
For datasets where accumulators saturate at fraction~$f$, the
encode$+$accumulate cost drops by $1/f$ while the
$g$-training and codebook-training costs remain unchanged.
On HIGGS ($n{=}11$M), encoding 1.1M points instead of 11M
reduces the encode$+$accumulate phase from ${\sim}9$\,s to
${\sim}0.9$\,s, cutting total rebuild time and increasing the
expected speedup over SISA accordingly.

\subsection{Experimental Details}
\label{app:experimental_details}

\paragraph{Base model $g$.}
Linear probe (logistic regression for classification, ridge regression
for regression) trained on $\mathcal{D}^{\circ}$ via Optuna
(200 trials, TPE sampler) over learning rate ($[10^{-5}, 10^{-1}]$),
weight decay ($[10^{-6}, 10^{-2}]$), epochs ($\{10, 50, 100, 200\}$),
and batch size ($\{32, 64, 128, 256\}$), minimizing validation loss
on a 20\% stratified holdout of $\mathcal{D}^{\circ}$.
When $|\mathcal{D}^{\circ}| \geq 5{,}000$, we also evaluate a
one-hidden-layer MLP ($D \to 128 \to k$, ReLU, Adam lr$=$1e-3,
wd$=$1e-4, bs$=$256, early stopping patience$=$3 on validation loss) and report
whichever yields higher validation accuracy.

\paragraph{Residual Quantization.}
Faiss~\cite{johnson2019billion} (version 1.7.4)
\texttt{ResidualQuantizer} with
\texttt{Train\_default} (k-means++ initialization) and
\texttt{max\_beam\_size}$=$1 (greedy
nearest-centroid assignment), $L_2$ distance.
Codebook hyperparameters ($b$, $\ell$)
selected by Optuna (30 trials per mode, TPE sampler) over
$b \in \{64, 128, 256, 512, 1024, 2048\}$ and
$\ell \in \{2, 4, 6, 8, 10, 12, 16\}$, with blending weight
$\alpha \in [0, 2]$ (grid of 21 values) cross-validated on the same
stratified holdout as the base model, optimizing validation accuracy
(classification) or Pearson~$r$ (regression).
All embeddings are $L_2$-normalized before codebook training.

\paragraph{LDA preconditioning.}
$W = W_{\text{LDA}} \in \mathbb{R}^{(k{-}1) \times d}$: all $k{-}1$
discriminant directions from standard LDA (scikit-learn
\texttt{LinearDiscriminantAnalysis}, default solver) fitted on
$\mathcal{D}^{\circ}$ labels and embeddings.
This is analogous to Optimized Product
Quantization~\cite{ge2013optimized}, but uses label supervision
from $\mathcal{D}^{\circ}$ rather than reconstruction error.
Feature map: $\varphi(x) = L_2\text{-normalize}([f(x);\; W\,f(x)])$
(concatenation, not projection). When $\mathcal{D}^{\circ} = \varnothing$,
$W = I$ (no preconditioning).

\paragraph{Encoders.}
Vision: OpenAI CLIP ViT-L/14
(\texttt{openai/clip-vit-large-patch14}, 768-d).
Text: Sentence-Transformers
\texttt{sentence-transformers/all-MiniLM-L6-v2} (384-d).
Standard preprocessing per encoder (CLIP: 224$\times$224 center crop;
SBERT: default tokenizer, max 512 tokens).

\paragraph{Tabular encoder (LightGBM leaf contributions).}
For tabular datasets (Skin Segmentation, Adult, Covertype, California
Housing), we train a LightGBM model~\cite{ke2017lightgbm} on
$\mathcal{D}^{\circ}$ and extract per-tree leaf indices as the
embedding. Configuration: 200 trees, max depth 6, learning rate 0.1,
min data in leaf 20, bagging fraction 0.8, feature fraction 0.8,
\texttt{num\_leaves}$=$31. Each input is mapped to a 200-d binary
vector (one-hot leaf index per tree, concatenated), yielding
$D{=}200$. The LightGBM model is trained only on
$\mathcal{D}^{\circ}$ and frozen, satisfying the same independence
requirement as the CLIP/SBERT encoders. Leaf contributions are
extracted via \texttt{model.predict(X, pred\_leaf=True)}.

\paragraph{SISA.}
$S{=}5$ shards. Each shard trains the same model architecture as $g$
on $\mathcal{D}^{\circ} \cup \text{shard}_k$. Predictions are
majority-voted (classification) or averaged (regression).

\paragraph{Data splits.}
Each dataset is split into $\mathcal{D}^{\circ}$ ($n^{\circ}$ points
selected uniformly at random, stratified by class for classification),
$\mathcal{D}^{\bullet}$ (remainder minus test), and a held-out test
set. Datasets with standard splits (ImageNet, CIFAR-100, MNIST,
AG~News, SST-2, DBpedia) use their official test sets; all others
use a 25\% random stratified test set.
Results are averaged over 5 random splits (seeds 0--4); the seed
controls $\mathcal{D}^{\circ}$ selection, RQ $k$-means initialization,
and Optuna sampling.

\paragraph{Timing experiments.}
All timing comparisons are measured on the same GPU (NVIDIA A100),
same process, median of 10 runs. SISA deletion retrains
\texttt{nn.Linear}+SGD on
$\mathcal{D}^{\circ} \cup \text{affected shard}$ with the same
Optuna-selected hyperparameters.
Table~\ref{tab:rq_timing} reports the raw RQA data structure
update time (the $\mathcal{D}^{\bullet}$ component of the
expected cost in Table~\ref{tab:timing}).

\begin{table}[H]
\centering
\caption{RQA data structure update latency (median of 10 runs,
A100 GPU): encode the deleted point, decrement accumulators,
and re-materialize prediction tables.  This is the
$t_{\text{RQ}}$ component of the expected QSS cost
(Table~\ref{tab:timing}).}
\label{tab:rq_timing}
\smallskip
\begin{small}
\begin{tabular}{lrrrr}
\toprule
Dataset & $n$ & RQ (ms) & SISA (ms) & Speedup \\
\midrule
CalHousing   & 21K  &   3.9 &      994 &    257$\times$ \\
UTKFace      & 24K  &   5.7 &  68{,}622 & 12{,}096$\times$ \\
DTD          & 5.6K &   6.1 &  26{,}105 &  4{,}278$\times$ \\
Adult        & 45K  &   8.0 &   1{,}188 &    148$\times$ \\
Skin Seg.    & 245K &  10.9 &   1{,}025 &     94$\times$ \\
SST-2        & 68K  &  12.0 &   1{,}303 &    109$\times$ \\
CIFAR-100    & 60K  &  12.0 &  31{,}563 &  2{,}620$\times$ \\
MNIST        & 70K  &  13.0 &  21{,}614 &  1{,}661$\times$ \\
AG~News      & 128K &  23.2 &   4{,}924 &    213$\times$ \\
Covertype    & 581K &  88.4 &  10{,}495 &    119$\times$ \\
DBpedia      & 630K & 186.7 &  20{,}642 &    111$\times$ \\
CelebA Young & 1.3M & 389.3 & 527{,}271 &  1{,}354$\times$ \\
ImageNet     & 1.3M & 429.8 & 376{,}799 &    877$\times$ \\
\bottomrule
\end{tabular}
\end{small}
\end{table}

\subsection{Detailed Pipeline Flowchart}
\label{app:flowchart}

\begin{figure}[H]
\centering
\begin{tikzpicture}[
  scale=0.9,
  transform shape,
  node distance=0.7cm and 0.9cm,
  every node/.style={font=\small},
  frozen/.style={draw, minimum height=0.65cm, minimum width=2.0cm,
    align=center, fill=black!8},
  mutable/.style={draw, minimum height=0.65cm, minimum width=2.0cm,
    align=center, fill=white, line width=0.8pt},
  dataset/.style={draw, minimum height=0.65cm, minimum width=1.8cm,
    align=center, font=\small\bfseries, fill=black!15},
  arr/.style={->, >=stealth, semithick},
  darr/.style={->, >=stealth, semithick, dashed, draw=black!60},
  lbl/.style={font=\scriptsize, text=black!55, inner sep=1pt},
  cost/.style={font=\scriptsize\ttfamily, text=black!60},
]

\node[dataset] (D) at (0, 0) {$\mathcal{D}$};

\node[frozen, below left=0.9cm and 1.2cm of D] (D0)
  {$\mathcal{D}^{\circ}$ {\scriptsize(non-del.)}};
\node[mutable, below right=0.9cm and 1.2cm of D] (Db)
  {$\mathcal{D}^{\bullet}$ {\scriptsize(deletable)}};

\draw[arr] (D.south) -- ++(0,-0.25) -| (D0.north)
  node[pos=0.75, left, font=\scriptsize] {$\rho$};
\draw[arr] (D.south) -- ++(0,-0.25) -| (Db.north)
  node[pos=0.75, right, font=\scriptsize] {$1{-}\rho$};

\node[frozen, below left=1.0cm and 1.4cm of D0] (g)
  {base model $g$};
\node[frozen, below=1.0cm of D0] (lda)
  {encoder $f$};
\node[frozen, below right=1.0cm and 1.4cm of D0] (rq)
  {RQ codebook};

\draw[arr] (D0.south) -- ++(0,-0.25) -| (g.north);
\draw[arr] (D0.south) -- (lda.north);
\draw[arr] (D0.south) -- ++(0,-0.25) -| (rq.north);

\draw[black!40, thick]
  ([xshift=-0.3cm]g.north west) --
  ([xshift=-0.3cm]g.south west);
\node[lbl, left=0.35cm of g, text=black!45,
  font=\scriptsize\itshape, align=right]
  {schema\\(frozen)};

\node[frozen, below=1.2cm of lda] (enc)
  {Encode $\varphi(x)$};
\node[mutable, right=1.8cm of enc, line width=1.0pt] (acc)
  {Accumulators $\mathbf{T}, n, N$};

\draw[arr] (lda.south) -- (enc.north);
\draw[arr, rounded corners=3pt] (rq.south)
  |- ([yshift=0.6cm, xshift=0.5cm]enc.north)
  -- ([xshift=0.5cm]enc.north);
\draw[arr, rounded corners=3pt] (Db.south)
  |- ([yshift=0.3cm]enc.north east)
  node[lbl, pos=0.75, above] {all $\mathcal{D}$}
  -- (enc.north east);

\draw[arr] (enc.east) -- (acc.west)
  node[lbl, midway, above] {codes $+ \, y^r_i$};

\node[lbl, right=0.15cm of acc, text=black!45,
  font=\scriptsize\itshape, align=left]
  {content\\(mutable)};

\node[draw, rounded corners=3pt, fill=black!5,
  minimum height=0.65cm, minimum width=3.6cm,
  below=1.5cm of enc, xshift=0.5cm, align=center] (pred)
  {$\hat{y} = g(x) + \alpha \cdot \textsc{RQA.Predict}(x)$};


\draw[arr] (g.south) |- (pred.west);
\draw[arr] (acc.south) |- (pred.east);

\node[draw, minimum height=0.6cm, minimum width=1.6cm,
  below=1.8cm of acc, xshift=-0.5cm] (check)
  {$x_j \in \mathcal{D}^{\circ}$\,?};

\node[mutable, above right=0.1cm and 0.8cm of check,
  minimum width=1.5cm] (dec)
  {Decrement};
\node[mutable, right=0.4cm of dec,
  minimum width=1.5cm] (mat)
  {Materialize};

\draw[arr] (check.east) -- ++(0.3,0) |- (dec.west)
  node[lbl, pos=0.25, above] {No};
\draw[arr] (dec.east) -- (mat.west)
  node[cost, midway, below] {$O(\ell^2 b^2 k)$};
\node[cost] at ([yshift=-0.25cm]dec.south) {$O(\ell^2)$};

\node[frozen, below right=0.5cm and 1.2cm of check,
  minimum width=2.5cm] (reb)
  {Rebuild schema $+$ accum};

\draw[arr] (check.east) -- ++(0.3,0) |- (reb.west)
  node[lbl, pos=0.25, below] {Yes};
\node[cost] at ([yshift=-0.25cm]reb.south) {$O(n)$};

\draw[darr] (mat.north) -- ++(0,0.4) -| (acc.south east)
  node[lbl, pos=0.2, right] {};

\node[lbl, font=\scriptsize\itshape] at ([yshift=-0.25cm]mat.south)
  {prob.\ $1{-}\rho$};
\node[lbl, font=\scriptsize\itshape] at ([yshift=-0.55cm]reb.south)
  {prob.\ $\rho$};

\matrix[draw, below=2.8cm of pred, anchor=center,
  column sep=0.8cm, inner sep=4pt, rounded corners=2pt] {
  \node[frozen, minimum height=0.3cm, minimum width=0.5cm,
    inner sep=0pt] {};
  & \node[font=\scriptsize] {Frozen (indep.\ of deletable data)};
  & \node[mutable, minimum height=0.3cm, minimum width=0.5cm,
    inner sep=0pt] {};
  & \node[font=\scriptsize] {Mutable (updated on deletion)}; \\
};

\end{tikzpicture}
\caption{End-to-end QSS pipeline. A random fraction~$\rho$ of
$\mathcal{D}$ forms the non-deletable subset~$\mathcal{D}^{\circ}$,
which defines the \emph{schema}: encoder~$f$, base model~$g$, and
the RQ codebook---all frozen before any deletion request. All
of~$\mathcal{D}$ is encoded and accumulated into mutable sufficient
statistics (\emph{content}). To unlearn a point from
$\mathcal{D}^{\bullet}$ (probability $1{-}\rho$): decrement
accumulators in $O(\ell^2)$ and re-materialize in
$O(\ell^2 b^2 k)$---the schema is untouched. If the point is in
$\mathcal{D}^{\circ}$ (probability~$\rho$): a full rebuild is
required at cost~$O(n)$.}
\label{fig:pipeline_flowchart}
\end{figure}
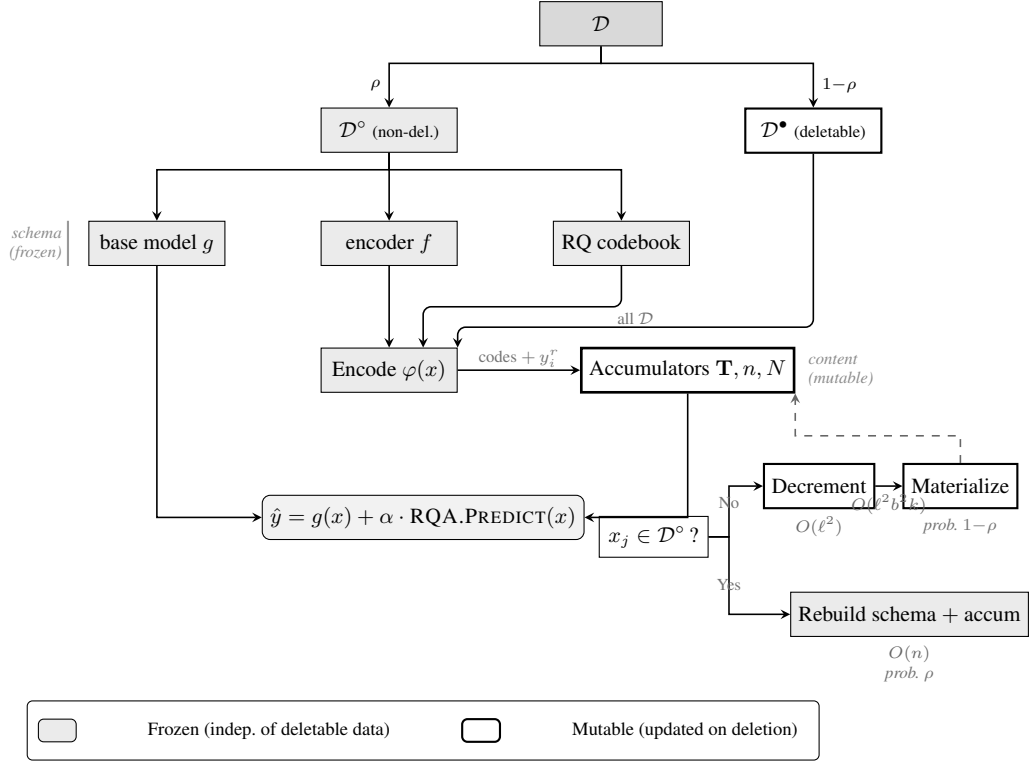

\section{Guarantee-Aligned and Analytic Comparisons}
\label{app:response_analysis}

\subsection{Fixed QSS-L and QSS-E results}
\label{app:mode_separation}

Table~\ref{tab:fixed_modes} uses one fixed construction per guarantee rather
than selecting between QSS-E and QSS-L using test accuracy. QSS-E trains its
codebook on the retained schema; QSS-L uses the fixed primary label-deletion
construction for each feature pipeline. QSS-L is more accurate on all 15 tasks
(median difference $5.27$ pp), quantifying the cost of requiring the codebook to
be independent of every deletable input.

\begin{table}[h]
\centering
\caption{Fixed-mode accuracy (\%, mean $\pm$ standard deviation over five
seeds) at $\rho=0.5\%$. The required neighboring relation, not test accuracy,
determines the mode.}
\label{tab:fixed_modes}
\small
\setlength{\tabcolsep}{3.5pt}
\begin{tabular}{lrr@{\qquad}lrr}
\toprule
Dataset & QSS-E & QSS-L & Dataset & QSS-E & QSS-L \\
\midrule
MNIST & $81.83{\pm}1.42$ & $94.43{\pm}0.34$ &
FMNIST & $75.03{\pm}1.56$ & $82.49{\pm}1.40$ \\
CelebA Young & $89.24{\pm}0.10$ & $89.34{\pm}0.52$ &
CelebA Male & $99.21{\pm}0.05$ & $99.25{\pm}0.05$ \\
SST-2 & $79.79{\pm}0.65$ & $85.36{\pm}1.97$ &
AG News & $84.62{\pm}0.53$ & $89.89{\pm}0.42$ \\
DBpedia & $91.24{\pm}0.97$ & $95.80{\pm}0.42$ &
Jigsaw & $92.03{\pm}0.05$ & $92.35{\pm}0.06$ \\
ImageNet & $46.24{\pm}0.31$ & $66.19{\pm}1.92$ &
CIFAR-100 & $39.03{\pm}2.55$ & $68.42{\pm}4.86$ \\
Skin Seg. & $99.46{\pm}0.14$ & $99.71{\pm}0.08$ &
Airlines & $57.80{\pm}0.17$ & $64.02{\pm}0.22$ \\
MiniBooNE & $86.88{\pm}0.21$ & $91.24{\pm}0.37$ &
HIGGS & $63.52{\pm}0.05$ & $72.98{\pm}0.08$ \\
KDD Cup & $99.30{\pm}0.03$ & $99.88{\pm}0.01$ & & & \\
\bottomrule
\end{tabular}
\end{table}

For QSS-L, we additionally deleted batches of 1, 10, 100, and 1,000
labels on AG~News. In every case, the source statistics, materialized cell
values, and prediction scores were bitwise identical to rebuilding from
$\mathcal{D}^{-y}$; prediction disagreement was zero.

\subsection{Exact analytic heads}
\label{app:analytic}

Ridge regression on a fixed representation is an important exact baseline. It
stores $A=X^\top X+\lambda I$ and $B=X^\top Y$; deleting $(x,y)$ applies
$A\leftarrow A-xx^\top$ and $B\leftarrow B-xy^\top$, followed by a
Sherman--Morrison downdate. We reconstructed the five submitted AG~News and
SST-2 splits, selected $\lambda$ using only an internal split of each training
set, and then fit the final head using all training labels.

\begin{table}[h]
\centering
\caption{Exact ridge and fixed QSS-L on identical frozen-SBERT features and
submitted splits. The QSS-L advantage on SST-2 is not statistically significant.}
\label{tab:ridge_comparison}
\small
\begin{tabular}{lrrrr}
\toprule
Dataset & Exact ridge & QSS-L & Difference & Paired $p$ \\
\midrule
AG News & $89.468$ & $89.894$ & $+0.426$ pp & $0.048$ \\
SST-2   & $83.709$ & $85.356$ & $+1.647$ pp & $0.121$ \\
\bottomrule
\end{tabular}
\end{table}

The Sherman--Morrison result agrees with a solve from rebuilt statistics to at
most $1.11\times10^{-12}$ coefficient error on AG~News and
$1.10\times10^{-13}$ on SST-2. Neither family dominates computationally:

\begin{table}[h]
\centering
\caption{State and ordinary-deletion costs, both independent of dataset size
$n$. Here $d$ is analytic feature dimension and $K$ the number of outputs.}
\label{tab:analytic_complexity}
\small
\begin{tabular}{lcc}
\toprule
Method & State & Ordinary deletion \\
\midrule
Ridge/ACU & $O(d^2+dK)$ & $O(d^2+dK)$ \\
QSS & $O(\ell^2b^2+\ell bK)$ & $O(\ell^2b^2K)$ conservative \\
\bottomrule
\end{tabular}
\end{table}

\subsection{Orthogonal schema and correction choices}
\label{app:orthogonal_choices}

QSS combines two design decisions that need not be made together.  The first is
\emph{what target to model}: one may fit the target directly on a fixed
representation, or first fit a retained-schema predictor $g$ and model the
residual $y-g(x)$.  This choice determines whether the correction depends on a
sampled schema.  Whenever $g$ changes after a schema deletion, every residual
defined relative to that copy of $g$ must be rebuilt, independently of how the
residual is stored.

The second decision is \emph{how to represent the mutable correction}.  A
linear or ridge correction stores global analytic sufficient statistics such
as $X^\top X$ and $X^\top Y$, as in ACU and the centralized construction of
Quan et al.~\cite{tang2025acu,quan2026exact}.  QSS instead stores local sums and
counts indexed by a frozen quantizer.  This choice determines the hypothesis
class and the state/update costs, rather than whether a schema is present.

Consequently, a standalone analytic head has no sampled-schema event; a QSS
predictor uses a schema-relative quantized residual; and a composed method can
use a schema predictor with an analytic residual head.  In the composed case,
each replicated schema owns its base predictor and residual statistics, so a
schema hit invalidates only that copy.  Analytic heads are preferable when a
modest-dimensional linear correction is expressive, while quantized statistics
provide local nonlinear capacity without introducing an explicit high-dimensional
feature expansion.

\subsection{Larger SISA shard counts}
\label{app:large_s}

Increasing $S$ can make an affected shard faster to retrain, but it does not
remove the coupling between deletion work and data per constituent model.
Table~\ref{tab:large_s} holds each model architecture fixed while extending
SISA through $S=200$.

\begin{table}[h]
\centering
\caption{SISA accuracy and isolated affected-shard retraining time.}
\label{tab:large_s}
\small
\begin{tabular}{lccc}
\toprule
Dataset & $S=5$ & $S=50$ & $S=200$ \\
\midrule
Adult & $83.30\%$ / $6.1$ ms & $83.32\%$ / $2.2$ ms & $81.91\%$ / $1.2$ ms \\
Covertype & $71.99\%$ / $3.379$ s & $71.46\%$ / $0.160$ s & $70.60\%$ / $0.0256$ s \\
SST-2 & $80.73\%$ / $0.1947$ s & $79.82\%$ / $0.1050$ s & $78.56\%$ / $0.0035$ s \\
\bottomrule
\end{tabular}
\end{table}

\subsection{Schema capacity and retained-set size}
\label{app:capacity}

The schema head should be selected jointly with~$\rho$. At the intended
$\rho=0.5\%$ regime, additional depth is often variance-limited: linear,
one-layer, and three-layer heads obtain $81.06/75.72/78.39\%$ on Adult,
$71.17/69.53/71.82\%$ on Covertype, and $72.40/73.47/74.46\%$ on SST-2.
Three-layer training is $8.2\times$ the linear-head time on Covertype and
$3.7\times$ on SST-2.

More retained data can make depth useful, at the cost of a larger schema. On
Covertype, the three-layer advantage over linear grows from $0.65$ pp at
$\rho=0.5\%$ to $5.75$ pp at $2\%$ and $10.74$ pp at $5\%$; its training time
at $5\%$ is $14.65$ s versus $0.86$ s for linear. At $5\%$, corresponding
three-seed QSS-L predictors obtain $72.90/77.20/83.83\%$ with linear,
one-layer, and three-layer schemas. The pattern is not universal: SST-2 favors
three layers at $0.5\%$, one layer at $2\%$, and one/two layers are effectively
tied at $5\%$. This is the practical scope boundary: richer schemas can improve
accuracy, but increasing~$\rho$ raises both rebuild probability and rebuild
cost.

\section{Unsupervised Quality Diagnostics}
\label{app:quality_unsup}

\begin{table}[H]
\centering
\caption{Without preconditioning, $\eta_0$ predicts quality.}
\smallskip
\begin{small}
\begin{tabular}{lcc}
\toprule
Attribute & $\eta_0$ & PR-AUC \\
\midrule
Male       & 0.960 & 0.998 \\
Young      & 0.532 & 0.959 \\
Attractive & 0.409 & 0.832 \\
Smiling    & 0.237 & 0.887 \\
Eyeglasses & 0.210 & 0.851 \\
\bottomrule
\end{tabular}
\end{small}
\end{table}

\subsection{Cell Purity Across Modalities}
\label{app:cell_purity}

Figure~\ref{fig:umap_purity} shows UMAP projections of test-set
embeddings from four datasets, shaded by level-0 cell purity (the
fraction of each cell's points belonging to the majority class).
MNIST (CLIP, $k{=}10$) achieves 96.6\% average purity at level~0:
the unsupervised RQ codebook already separates classes well, which
explains QSS's strong performance. At the other extreme, Covertype
($k{=}7$, $b{=}64$) reaches only 63.1\%.
CIFAR-100 ($k{=}100$) and AG~News ($k{=}4$) fall between.

Multi-level RQ drives effective purity above 99.8\% on all four
datasets (level-0 impurity is corrected by deeper levels), but the
per-level improvement depends on the base purity: when level~0 is
already high, fewer levels suffice.

\begin{figure}[H]
\centering
\includegraphics[width=\linewidth]{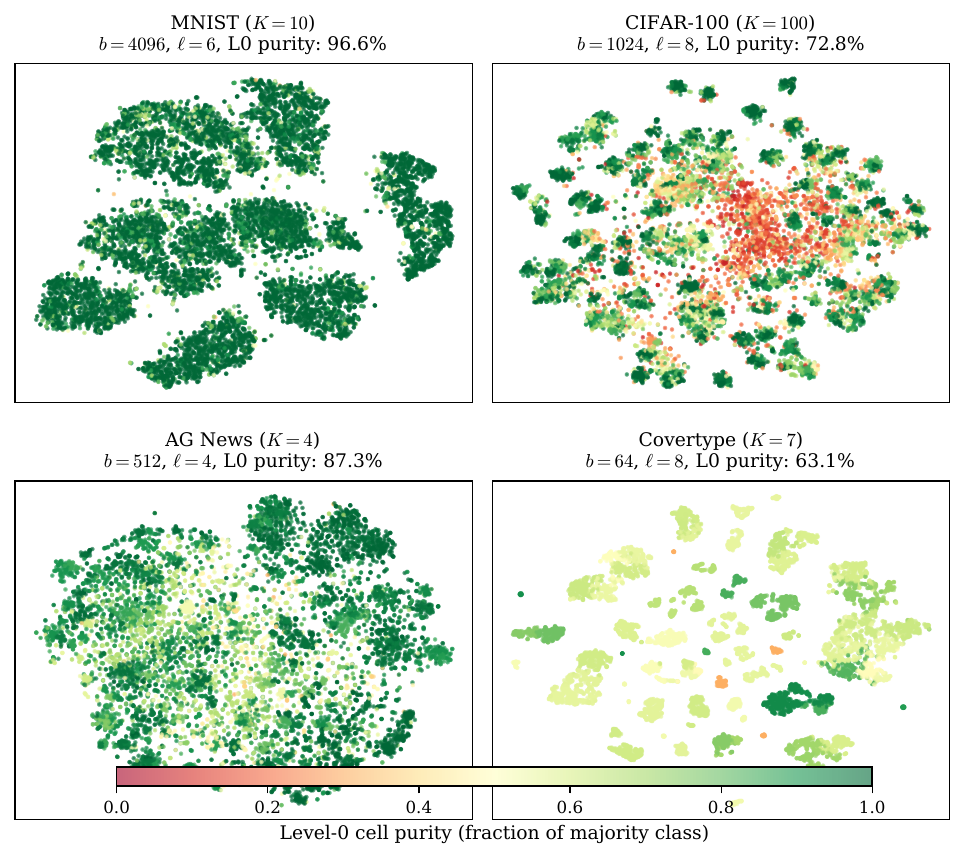}
\caption{UMAP projections of test-set embeddings, shaded by level-0
cell purity (fraction of each cell belonging to the majority class).
Dark regions indicate pure cells ($>$90\%); light regions indicate
mixed cells.
MNIST (96.6\% average purity) achieves the smallest QSS-L--SISA gap
($+3.8$\,pp); Covertype (63.1\%) has the largest among low-$k$
tasks ($-10.7$\,pp).}
\label{fig:umap_purity}
\end{figure}


\newpage
\section*{NeurIPS Paper Checklist}

\begin{enumerate}

\item {\bf Claims}
    \item[] Question: Do the main claims made in the abstract and introduction accurately reflect the paper's contributions and scope?
    \item[] Answer: \answerYes{}
    \item[] Justification: The abstract and introduction distinguish exact label deletion (QSS-L) from exact instance deletion (QSS-E), state the randomized rebuild event, and scope the empirical claims to the corresponding construction. Theoretical results appear in \S3 and guarantee-aligned experiments in \S4 and the appendices.

\item {\bf Limitations}
    \item[] Question: Does the paper discuss the limitations of the work performed by the authors?
    \item[] Answer: \answerYes{}
    \item[] Justification: \S4.6 discusses three limitations: the high-$k$ accuracy gap to SISA, codebook capacity and depth collapse, and the Jigsaw negative result. \S5 discusses the $O(N)$ rebuild cost when deletions hit $\mathcal{D}^{\circ}$, adversarial targeting of $\mathcal{D}^{\circ}$, and the linear baseline choice.

\item {\bf Theory assumptions and proofs}
    \item[] Question: For each theoretical result, does the paper provide the full set of assumptions and a complete (and correct) proof?
    \item[] Answer: \answerYes{}
    \item[] Justification: Propositions 1--3 state assumptions explicitly (independence of codebook from targets, Laplace noise scale, variance-explained ratio). Full proofs are in Appendices A.1--A.3. The $O(\ell/n_{\min})$ tighter DP bound is clearly marked as an open conjecture with empirical evidence.

\item {\bf Experimental result reproducibility}
    \item[] Question: Does the paper fully disclose all the information needed to reproduce the main experimental results of the paper to the extent that it affects the main claims and/or conclusions of the paper (regardless of whether the code and data are provided or not)?
    \item[] Answer: \answerYes{}
    \item[] Justification: Algorithm~1 specifies the data structure. \S4.1 describes all datasets, encoders, data splits, baselines, and hyperparameter selection. All datasets are publicly available.

\item {\bf Open access to data and code}
    \item[] Question: Does the paper provide open access to the data and code, with sufficient instructions to faithfully reproduce the main experimental results, as described in supplemental material?
    \item[] Answer: \answerYes{}
    \item[] Justification: Code is available at \url{https://github.com/atavory/rqu}. All datasets used are publicly available (CelebA, ImageNet, CIFAR-100, Covertype, etc.).

\item {\bf Experimental setting/details}
    \item[] Question: Does the paper specify all the training and test details (e.g., data splits, hyperparameters, how they were chosen, type of optimizer) necessary to understand the results?
    \item[] Answer: \answerYes{}
    \item[] Justification: \S4.1 specifies data splits ($|\mathcal{D}^{\circ}|$, random selection), base model architecture (linear probe or MLP with stated dimensions), cross-validation procedure for $(b, \ell, \alpha)$, and SISA configuration ($S{=}5$).

\item {\bf Experiment statistical significance}
    \item[] Question: Does the paper report error bars suitably and correctly defined or other appropriate information about the statistical significance of the experiments?
    \item[] Answer: \answerYes{}
    \item[] Justification: Main results are 5-seed means over random train/test/$\mathcal{D}^{\circ}$ splits. The number of seeds is stated in the table captions.

\item {\bf Experiments compute resources}
    \item[] Question: For each experiment, does the paper provide sufficient information on the computer resources (type of compute workers, memory, time of execution) needed to reproduce the experiments?
    \item[] Answer: \answerYes{}
    \item[] Justification: Timing experiments report GPU type (A100) and wall-clock times in Table~2. Training and deletion costs are reported in milliseconds.

\item {\bf Code of ethics}
    \item[] Question: Does the research conducted in the paper conform, in every respect, with the NeurIPS Code of Ethics \url{https://neurips.cc/public/EthicsGuidelines}?
    \item[] Answer: \answerYes{}
    \item[] Justification: The research addresses data privacy (right to deletion) and uses only publicly available datasets. No human subjects or personally identifiable data were collected.

\item {\bf Broader impacts}
    \item[] Question: Does the paper discuss both potential positive societal impacts and negative societal impacts of the work performed?
    \item[] Answer: \answerYes{}
    \item[] Justification: \S5 discusses positive impact (enabling GDPR compliance at scale) and broader impact (codebook instability affects VQ-VAE, neural codecs, and retrieval systems). The adversarial $\mathcal{D}^{\circ}$ targeting risk is acknowledged.

\item {\bf Safeguards}
    \item[] Question: Does the paper describe safeguards that have been put in place for responsible release of data or models that have a high risk for misuse (e.g., pre-trained language models, image generators, or scraped datasets)?
    \item[] Answer: \answerNA{}
    \item[] Justification: The paper does not release pre-trained models, generative models, or scraped datasets. QSS is an unlearning method, not a generative system.

\item {\bf Licenses for existing assets}
    \item[] Question: Are the creators or original owners of assets (e.g., code, data, models), used in the paper, properly credited and are the license and terms of use explicitly mentioned and properly respected?
    \item[] Answer: \answerYes{}
    \item[] Justification: All datasets and models are cited (CLIP, SBERT, Faiss, CelebA, ImageNet, etc.). All are publicly available under their respective licenses.

\item {\bf New assets}
    \item[] Question: Are new assets introduced in the paper well documented and is the documentation provided alongside the assets?
    \item[] Answer: \answerYes{}
    \item[] Justification: The QSS implementation is available at \url{https://github.com/atavory/rqu} with documentation and scripts to reproduce the main experiments.

\item {\bf Crowdsourcing and research with human subjects}
    \item[] Question: For crowdsourcing experiments and research with human subjects, does the paper include the full text of instructions given to participants and screenshots, if applicable, as well as details about compensation (if any)?
    \item[] Answer: \answerNA{}
    \item[] Justification: No crowdsourcing or human subjects research was conducted.

\item {\bf Institutional review board (IRB) approvals or equivalent for research with human subjects}
    \item[] Question: Does the paper describe potential risks incurred by study participants, whether such risks were disclosed to the subjects, and whether Institutional Review Board (IRB) approvals (or an equivalent approval/review based on the requirements of your country or institution) were obtained?
    \item[] Answer: \answerNA{}
    \item[] Justification: No human subjects research was conducted.

\item {\bf Declaration of LLM usage}
    \item[] Question: Does the paper describe the usage of LLMs if it is an important, original, or non-standard component of the core methods in this research? Note that if the LLM is used only for writing, editing, or formatting purposes and does \emph{not} impact the core methodology, scientific rigor, or originality of the research, declaration is not required.
    \item[] Answer: \answerNA{}
    \item[] Justification: LLMs were used only for writing assistance and editorial polishing. They are not a component of the QSS method or experimental pipeline.

\end{enumerate}

\end{document}